\documentclass{article}
\usepackage{dsfont}    
\usepackage{amsmath}
\DeclareMathOperator*{\argmin}{arg\,min}
\DeclareMathOperator*{\argmax}{arg\,max}
\newcommand{\E}{\mathbb{E}}
\allowdisplaybreaks

\usepackage{microtype}
\usepackage{graphicx}
\usepackage{subcaption}
\usepackage{booktabs} 

\usepackage{hyperref}

\usepackage[preprint]{icml2026}

\usepackage{amssymb}
\usepackage{mathtools}
\usepackage{amsthm}

\usepackage[capitalize,noabbrev]{cleveref}

\theoremstyle{plain}
\newtheorem{theorem}{Theorem}[section]

\newtheorem{lemma}[theorem]{Lemma}

\theoremstyle{definition}

\newtheorem{assumption}[theorem]{Assumption}
\theoremstyle{remark}

\usepackage[disable,textsize=tiny]{todonotes}

\icmltitlerunning{Linear and Neural Dueling Bandits with Delayed Feedback}

\begin{document}

\twocolumn[
  \icmltitle{Linear and Neural Dueling Bandits with Delayed Feedback}




\begin{icmlauthorlist}
  \icmlauthor{Xiangyi Wang}{cuhksz}
  \icmlauthor{Pingchen Lu}{cuhksz}
  \icmlauthor{Jie Mao}{cuhksz}
  \icmlauthor{Mingze Kong}{cuhksz}
  \icmlauthor{Zhi Hong}{cuhksz}
  \icmlauthor{Zhiyong Wang}{cuhk}
  \icmlauthor{Zhongxiang Dai}{cuhksz}
\end{icmlauthorlist}

\icmlaffiliation{cuhksz}{The Chinese University of Hong Kong, Shenzhen}
\icmlaffiliation{cuhk}{The Chinese University of Hong Kong}

\icmlcorrespondingauthor{Zhiyong Wang}{zhiyongwang@cuhk.edu.hk}
\icmlcorrespondingauthor{Zhongxiang Dai}{daizhongxiang@cuhk.edu.cn}

  \icmlkeywords{Machine Learning, ICML}

  \vskip 0.3in
]



\newcommand{\zhiyong}[1]{{\color{blue} [\textbf{Zhiyong:} \textnormal{#1}]}}
\printAffiliationsAndNotice{}  
\begin{abstract}Contextual dueling bandits form a cornerstone of preference-based decision-making, with critical applications in recommender systems and large language model alignment. However, standard algorithms rely on the idealized assumption of immediate feedback, a condition frequently violated in real-world scenarios such as prompt optimization. This setting introduces a unique theoretical challenge: unlike linear bandits, dueling bandit estimators lack closed-form solutions, rendering naive adaptations of standard weighting techniques biased. To address this, we formalize the problem of Contextual Dueling Bandits with Stochastic Delayed Feedback and propose two novel algorithms: Linear (LDB-DF) and Neural (NDB-DF) Dueling Bandits with Delayed Feedback. Central to our approach is a novel estimator that integrates an Inverse Probability Weighting (IPW) mechanism directly into the loss function, ensuring unbiased correction for delayed or missing feedback. We provide comprehensive theoretical analysis, establishing an $\tilde{O}(d\sqrt{T})$ regret bound for the linear setting and sub-linear guarantees for the neural setting. Extensive experiments on both simulated and real-world datasets demonstrate the effectiveness of our proposed methods.\end{abstract}

\section{Introduction}
\label{sec:intro}

The contextual multi-armed bandit (MAB) framework serves as a cornerstone of sequential decision-making under uncertainty, finding widespread application in recommendation systems, online advertising, and personalized services \citep{Li2010, Lattimore2020}. While traditional MAB algorithms predicate on absolute numerical rewards, many real-world scenarios---such as calibrating Large Language Models (LLMs) or ranking search results---fundamentally rely on relative preference feedback. In these settings, assigning an absolute scalar score to a response is often cognitively burdensome and noisy for human annotators. In contrast, preference-based feedback, which requires an annotator to simply compare two candidate responses (e.g., $x_1 \succ x_2$), is widely recognized as being more reliable and consistent \citep{Yue2012, Yue2009}. This paradigm is formally captured by the Contextual Dueling Bandit (CDB) framework \citep{Yue2012, Saha2021, Bengs2022}, where the agent selects a pair of arms and updates its policy based on binary preference outcomes.

Despite the theoretical success of CDBs, a critical aspect of real-world deployment remains largely unaddressed: the temporal dynamics of feedback. Standard CDB algorithms typically operate under the assumption that preference feedback is observed immediately following an action. However, in many practical applications, feedback is inherently delayed. For instance, in display advertising and e-commerce, user reactions such as conversions or clicks can lag behind impressions by hours or even days \citep{Chapelle2014}. Similarly, in the optimization of LLMs via human-in-the-loop reinforcement learning (RLHF), human evaluation is asynchronous and slow \citep{lin2024prompt}. Furthermore, feedback systems often impose a ``waiting window''; if a response does not arrive within this window, it is treated as missing data \citep{Vernade2017}.

Ignoring these delays introduces significant bias into reward estimation. While stochastic delayed feedback has been studied in standard MABs and Bayesian Optimization (BO), the unique mathematical structure of dueling bandits renders existing solutions insufficient. For instance, in delayed BO, \citet{Verma2022BO} successfully employed a ``weighting'' strategy where pending observations are temporarily imputed with minimum values. Such approaches are often feasible in linear or kernelized settings where estimators admit closed-form solutions (e.g., ridge regression), allowing the bias from imputation to be theoretically bounded. However, this presents a unique theoretical challenge for preference learning: parameter estimation in contextual dueling bandits---typically formulated via Maximum Likelihood Estimation (MLE) under the Bradley-Terry-Luce model---lacks a closed-form solution and exhibits complex, implicit dependencies on binary outcomes. Consequently, naively adapting standard weighting techniques induces a systematic bias in the optimization landscape that cannot be easily rectified.

To bridge this gap, we formalize the problem of Contextual Dueling Bandits with Stochastic Delayed Feedback and propose a principled solution that departs from heuristic imputation. We introduce two novel algorithms: Linear Dueling Bandits with Delayed Feedback (LDB-DF) for linear rewards, and Neural Dueling Bandits with Delayed Feedback (NDB-DF) to capture complex, non-linear reward landscapes. Central to our approach is a novel estimator that integrates an Inverse Probability Weighting (IPW) mechanism directly into the regularized loss function. By scaling the contribution of observed feedback by the inverse of its arrival probability, this mechanism explicitly ensures that the expected gradient of the surrogate loss remains an unbiased estimator of the true gradient, effectively neutralizing the bias introduced by stochastic observation. We provide a comprehensive theoretical analysis, establishing an $\tilde{O}(d\sqrt{T})$ regret bound for the linear setting and proving sub-linear regret guarantees for the neural setting via Neural Tangent Kernel (NTK) analysis. Extensive experiments on both simulated environments and real-world datasets demonstrate the superior performance of our methods compared to baselines that ignore feedback delays.

\section{Problem Setting}
\label{sec:problem_setting}

 Throughout this paper, we denote vectors using bold lowercase letters (e.g., $\mathbf{x}$) and matrices using bold uppercase letters (e.g., $\mathbf{M}$). Let $|\mathcal{A}|$ represent the cardinality of a set $\mathcal{A}$, and $[m]$ denote the integer set $\{1, 2, \ldots, m\}$. Given a positive semi-definite (PSD) matrix $\mathbf{M}$, we define the weighted matrix norm of a vector $\mathbf{x}$ as $\|\mathbf{x}\|_\mathbf{M} = \sqrt{\mathbf{x}^\top \mathbf{M} \mathbf{x}}$.

 Consider a sequential decision-making scenario where a learning agent interacts with an environment over $T$ rounds. At each time step $t \in [T]$, the agent receives a set of arms $\mathcal{X}_t$ and selects a pair $(x_{t,1}, x_{t,2})$. In a departure from standard bandit settings that provide scalar rewards, we operate under the dueling bandit paradigm, where the agent only receives a binary preference feedback $y_t \in \{0, 1\}$ indicating which arm in the pair is superior. Distinct from standard settings, the preference feedback $y_t$ is not immediately available; instead, it is observed after a random delay $D_t$.

We impose a hard observation threshold $M$. Under this setting, the feedback arrives within $M$ rounds with probability $\rho = \mathbb{P}(D_t \le M)$, while it arrives outside this window (i.e., $D_t > M$) with probability $1-\rho$. Following standard assumptions in delayed feedback literature \citep{Vernade2017, Garg2020}, we assume $\rho$ is known. Consequently, from the agent's perspective, at any given round $t$, we can only utilize the historical data that has already arrived, i.e., $s + D_s \le t$.

\textbf{Modeling Preference Feedback.} We assume the existence of a latent reward function $f(x)$. For any pair of arms $x$ and $x'$, the pairwise preference feedback $y = \mathds{1}(x \succ x')$ is generated by the Bradley-Terry-Luce (BTL) model \citep{Wang2025Clustering}. Specifically, $y$ is sampled from a Bernoulli distribution with probability:
\[
P(x \succ x') = \mu(f(x) - f(x')),
\]
where $\mu(z) = \frac{1}{1+e^{-z}}$ is the logistic link function.

\textbf{Performance Metric.} Consistent with standard contextual dueling bandit literature \citep{Bengs2022, Saha2022}, we evaluate the algorithm using cumulative regret:
\[
    R_T = \sum_{t=1}^T r_t = \sum_{t=1}^T \left( 2f(x^*_t) - f(x_{t,1}) - f(x_{t,2}) \right),
\]
where $x^*_t = \argmax_{x \in \mathcal{X}_t} f(x)$ denotes the optimal arm in round $t$.

\noindent We adopt the following standard assumptions regarding the preference model \citep{Yue2012,Saha2021,Bengs2022,Wang2025Clustering}: %

\begin{assumption}[Standard Assumptions]
\label{ass:standard}
\mbox{}\\
$A1. \ |\mu(f(x)) - \mu(g(x))| \le L_\mu |f(x) - g(x)|$. \\
$A2. \ \forall x \in \mathcal{X}: \min |\dot{\mu}(f(x))| \ge \kappa_\mu > 0$. \\
$A3. \ \|\phi(x_{t,1}) - \phi(x_{t,2})\|_2 \leq L, \forall t = 1, \dots, T$.
\end{assumption}
\noindent The notation $\dot{\mu}(\cdot)$ denotes the first-order derivative of the link function $\mu(\cdot)$. The assumptions listed above are widely adopted in studies involving binary feedback and linear structures. For the specific case where $\mu$ is the logistic function, the Lipschitz constant is $L_\mu = 1/4$. Regarding Assumption (A3), while classical linear bandit analysis typically assumes a bound on the norm of individual feature vectors (i.e., $\|\phi(x)\| \le L$), our condition on the pairwise difference is structurally equivalent. By the triangle inequality, the standard assumption implies ours with a scaling factor of $2L$.

\textbf{Reward Function Realizations.} Following established conventions in contextual bandits \citep{Zhou2020, Zhang2021}, we consider two realizations for the latent reward function $f$:
\subsection{Linear Dueling Bandits With Delayed Feedback}

In the linear setting, we posit that the reward function is linear in a fixed feature space, $f(x) = \theta^\top \phi(x)$, where the feature map $\phi : \mathbb{R}^{d'} \to \mathbb{R}^d$ satisfies $\|\phi(x)\|_2 \le 1$. Under this assumption, the preference probability simplifies to:
\[
    P(x_{t,1} \succ x_{t,2}) = \mu\left(\theta^\top (\phi(x_{t,1}) - \phi(x_{t,2}))\right).
\]
The objective is to estimate the unknown parameter $\theta \in \mathbb{R}^d$. However, the estimator must account for the observing mechanism: at any round $t$, we can only utilize feedback where  $s + D_s \le t$.

\subsection{Neural Dueling Bandits With Delayed Feedback}
In the neural setting,we allow the reward functions $f$ to be non-linear functions.
We approximate $f(x)$ using a fully connected neural network (NN) $h(x; \theta)$ with depth $L \ge 2$, width $m$, and ReLU activations. The network output is defined as:
\[
h(x; \theta) = \mathbf{W}_L \text{ReLU} (\mathbf{W}_{L-1} \text{ReLU} (\cdots \text{ReLU} (\mathbf{W}_1 x))),
\]
where $\mathbf{W}_l$ represents the weight matrices. The parameter vector $\theta$ comprises all flattened weights, and $g(x; \theta) = \nabla_\theta h(x; \theta)$ denotes the gradient.

Analogous to the linear case, the primary challenge lies in training the network using only the subset of feedback available within the observation window. To facilitate theoretical analysis via the Neural Tangent Kernel (NTK) regime, we adopt standard assumptions from the neural bandit literature:
\begin{assumption}[Neural Assumptions]
\label{ass:neural_assumptions}
\mbox{}\\
$B1. \ |f(x)| \le 1, \forall x \in \mathcal{X}$. \\
$B2. \ H \succeq \lambda_0 I, \text{ for some } \lambda_0 > 0$. \\
$B3. \ \forall x \in \mathcal{X}: \|x\|_2 = 1 \text{ and } x_j = x_{j+d/2}$.
\end{assumption}
\noindent The last assumption in Assumption \ref{ass:neural_assumptions}, together with the initialization of $\theta_0$, follows the standard practice in neural bandits \citep{Zhou2020, Zhang2021}, which ensures that $h(x; \theta_0) = 0$ for all $x \in \mathcal{X}_t, t \in [T]$. Moreover, the assumption $x_j = x_{j+d/2}$ is a mild condition commonly used in the neural bandit literature \citep{Zhou2020, Zhang2021} and has been recently adopted in extensions to dueling bandits \citep{Verma2025NDB}. This assumption is primarily for the convenience of regret analysis: for any context $x$ with $\|x\|_2 = 1$, one can always construct a padded context $x' = (x^\top, x^\top)^\top/\sqrt{2}$ that satisfies this requirement \citep{Zhou2020}.

\section{Algorithms}
\label{sec:algorithms}

\subsection{Linear Dueling Bandits with Delayed Feedback (LDB-DF)}

To address the challenge of parameter estimation under stochastic delays, we introduce the Linear Dueling Bandits with Delayed Feedback (LDB-DF) algorithm. The core difficulty in this setting is that naively utilizing available data introduces systematic bias. LDB-DF overcomes this by integrating an unbiased estimator with an optimistic exploration strategy. Specifically, the algorithm relies on an Inverse Probability Weighting (IPW) mechanism to correct the gradient bias in the Maximum Likelihood Estimation.

\textbf{Estimating Reward Parameters with Observed Data (Line \ref{line:ldb_update_set}-Line \ref{line:ldb_compute_theta}).}
Unlike standard approaches that update the model immediately, LDB-DF maintains an estimator derived exclusively from the subset of historical data that has reliably "arrived" within the valid window. At the onset of round $t$, the algorithm first updates the set of observed feedback $\mathcal{O}_t$ (Line~\ref{line:ldb_update_set}).

To formalize the observation mechanism, we define the effective observable feedback $y'_s$ for a past round $s$. From the learner's perspective, any feedback with a delay exceeding the threshold $M$ is strictly treated as zero:
\begin{equation}
    y'_s = y_s \mathds{1}\{D_s \le M\}.
    \label{eq:effective_y}
\end{equation}
Under this model, the observation is governed by the relation $y_s' = \mu(f(x_{s,1}) - f(x_{s,2}))\rho + \epsilon_s$, where $\epsilon_s$ captures the stochasticity of both the preference generation and the random observation. Consequently, a naive Maximum Likelihood Estimation (MLE) update restricted to these raw observations would suffer from significant selection bias: since the effective sample size is systematically reduced by the observation probability $\rho = \mathbb{P}(D_s \le M)$, the estimator would incorrectly treat missing data as non-informative zeros.

To rectify this, we introduce an Inverse Probability Weighting (IPW) strategy. We define the effective weight for the feedback from round $s$ observed at round $t$ as:
\begin{equation}
    \omega_{s,t} \triangleq \frac{\mathds{1}\{D_s \leq \min(M, t-s)\}}{\rho}.
    \label{eq:weight_def}
\end{equation}
By re-weighting observed feedback instances by $\omega_{s,t}$, we compensate for the mass of unobserved instances. Consequently, in Line \ref{line:ldb_compute_theta}, $\theta_t$ is determined by minimizing the following IPW-regularized loss function:
\begin{equation}
\begin{aligned}
    \mathcal{L}_t(\theta) = \;&- \sum_{s=1}^{t} \Big[ \omega_{s,t} y_s \log \mu ( \theta^\top \Delta\phi_s ) + \\
    &\quad \left(1 - \omega_{s,t} y_s\right) \log \mu ( -\theta^\top \Delta\phi_s ) \Big] + \frac{\lambda}{2} \|\theta\|_2^2,
\end{aligned}
\label{eq:ldb_loss}
\end{equation}
where $\Delta\phi_s \triangleq \phi(x_{s,1}) - \phi(x_{s,2})$. The weight $\omega_{s,t}$ effectively neutralizes the bias introduced by the observation mechanism where $D_s > M$.
\begin{algorithm*}[t]
\caption{Linear Dueling Bandits with Delayed Feedback (LDB-DF)}
\label{alg:ldb_df}
\begin{algorithmic}[1]
    \STATE \textbf{Input:}Confidence $\delta \in (0,1)$, Regularization $\lambda$, probability $\rho\in (0,1)$, delay threshold $M$. \label{line:ldb_input}
    \STATE \textbf{Initialize:} $V_0 \triangleq \frac{\lambda}{\kappa_\mu} \mathbf{I}$
    \STATE \textbf{Initialize:} History of arrived feedback $\mathcal{D} \leftarrow \emptyset$. \label{line:ldb_init_hist}
    \FOR{$t=1, \dots, T$} \label{line:ldb_loop_start}
        \STATE Update observed set: $\mathcal{O}_t = \{(s, x_{s,1}, x_{s,2}, y_s) \mid s < t \text{ and feedback arrives at } t\}$. \label{line:ldb_update_set}
        \STATE Compute $\theta_t = \argmin_{\theta} \mathcal{L}_t(\theta)$ via Eq. (\ref{eq:ldb_loss}) using currently available data. \label{line:ldb_compute_theta}
        \STATE Select first arm: $x_{t,1} = \argmax_{x \in \mathcal{X}_t} \theta_t^\top \phi(x)$. \label{line:ldb_select_arm1}
        \STATE Select second arm: $x_{t,2} = \argmax_{x \in \mathcal{X}_t} \left[ \theta_t^\top (\phi(x)-\phi(x_{t,1})) + \frac{\beta_t}{\kappa_\mu} \|\phi(x) - \phi(x_{t,1})\|_{V_{t-1}^{-1}} \right]$. \label{line:ldb_select_arm2}
        \STATE \textbf{Duel:} Play $(x_{t,1}, x_{t,2})$. Feedback $y_t$ is observed after delay $D_t$. \label{line:ldb_duel}
        \STATE Update Matrix: $V_t \leftarrow V_{t-1} + (\phi(x_{t,1}) - \phi(x_{t,2})) (\phi(x_{t,1}) - \phi(x_{t,2}))^\top$. \label{line:ldb_update_matrix}
    \ENDFOR \label{line:ldb_loop_end}
\end{algorithmic}
\end{algorithm*}
Simultaneously, LDB-DF maintains an aggregated information matrix $V_{t-1}$ (updated in Line \ref{line:ldb_update_matrix}) to quantify exploration uncertainty in the feature space:
\begin{equation}
    V_{t-1} = \frac{\lambda}{\kappa_\mu} \mathbf{I} + \sum_{s=1}^{t-1} (\phi(x_{s,1}) - \phi(x_{s,2}))(\phi(x_{s,1}) - \phi(x_{s,2}))^\top.
    \label{eq:ldb_matrix}
\end{equation}
It is crucial to note that while the parameter estimator $\theta_t$ relies on the \textit{arrived} feedback (weighted by $\omega_{s,t}$), the design matrix $V_{t-1}$ is updated using the features of \textit{all} selected pairs up to round $t-1$, as context vectors are observed immediately upon action selection.

\textbf{Arm Recommendation Strategy.}
Leveraging the estimated preference vector $\theta_t$ and the information matrix $V_{t-1}$, LDB-DF selects a pair of arms $(x_{t,1}, x_{t,2})$ from the current set $\mathcal{X}_t$ via the following strategy:
\begin{itemize}
    \setlength\abovedisplayskip{4pt}
    \setlength\belowdisplayskip{4pt}

    \item \textbf{First Arm Selection (Line \ref{line:ldb_select_arm1}).} The first arm is selected greedily to maximize the estimated utility based on the current model $\theta_t$:
    \begin{equation}
        x_{t,1} = \argmax_{x \in \mathcal{X}_t} \theta_t^\top \phi(x).
        \label{eq:ldb_arm1}
    \end{equation}

    \item \textbf{Second Arm Selection (Line \ref{line:ldb_select_arm2}).} The second arm is chosen to balance exploitation and exploration by maximizing an Upper Confidence Bound (UCB) relative to the first arm:
    \begin{equation}
    \begin{aligned}
        x_{t,2} = \argmax_{x \in \mathcal{X}_t} \Big[ &\theta_t^\top \big(\phi(x) - \phi(x_{t,1})\big) + \\
        &\frac{\beta_t}{\kappa_\mu} \big\|\phi(x) - \phi(x_{t,1})\big\|_{V_{t-1}^{-1}} \Big].
    \end{aligned}
    \label{eq:ldb_arm2}
    \end{equation}
    Intuitively, Eq. (\ref{eq:ldb_arm2}) encourages selecting an arm $x$ that either exhibits a strong predicted preference over $x_{t,1}$ (exploitation) or is sufficiently distinct from $x_{t,1}$ within the feature space defined by $V_{t-1}$ (exploration). The confidence radius $\beta_t$ accounts for uncertainty arising from both stochastic noise and feedback delays.
\end{itemize}

\textbf{Updating Interaction History (Line \ref{line:ldb_duel}).}
Upon recommending the pair $(x_{t,1}, x_{t,2})$, the environment generates a binary preference $y_t$ which is not immediately observable. Instead, it arrives after an unknown stochastic delay $D_t$ (Line \ref{line:ldb_duel}). Conceptually, this feedback enters a ``pending" state. It is only incorporated into the observed set $\mathcal{O}_{t'}$ for training in a future round $t'$ once the arrival condition is met (i.e., $t \ge s + D_s$) and provided the delay satisfies the observation constraint $D_t \le M$. Meanwhile, the information matrix $V_t$ is updated immediately using the context vectors of the selected arms (Line \ref{line:ldb_update_matrix}), as feature availability is independent of feedback latency.

\subsection{Neural Dueling Bandits with Delayed Feedback (NDB-DF)}

Recognizing the limitations of linear models in capturing complex reward landscapes, we propose Neural Dueling Bandits with Delayed Feedback (NDB-DF). This algorithm extends our IPW-based approach to non-linear function approximation using neural networks. NDB-DF maintains a neural approximator updated online, crucially learning only from the subset of effectively arrived feedback.

\textbf{Estimating Reward Model with Data (Line \ref{line:ndb_update_set}-Line \ref{line:ndb_train}).}
Similar to the linear setting, NDB-DF trains the neural network using only the subset of feedback that has reliably arrived by the current round $t$ and satisfies the observation constraint $D_s \le M$. To correct for the systematic bias introduced by this selective observation, we employ the same Inverse Probability Weighting (IPW) weights $\omega_{s,t}$ as defined in Eq. (\ref{eq:weight_def}).
Defining the predicted preference difference as $\Delta h_s(\theta) \triangleq h(x_{s,1}; \theta) - h(x_{s,2}; \theta)$, the network parameters $\theta_t$ are obtained by minimizing the following regularized, probability-weighted cross-entropy loss:
\begin{equation}
\begin{split}
    \mathcal{L}_t(\theta) = &- \sum_{s=1}^{t-1} \Bigg[ \omega_{s,t} y_s \log \mu (\Delta h_s(\theta)) \\
    &\qquad + \left(1 - \omega_{s,t} y_s\right) \log \mu (-\Delta h_s(\theta)) \Bigg] \\
    &+ \frac{\lambda}{2} \|\theta - \theta_0\|_2^2.
\end{split}
\label{eq:ndb_loss}
\end{equation}
where $\theta_0$ represents the initialized parameters. The term $\omega_{s,t}$ ensures the estimator remains unbiased despite stochastic observation. Additionally, we compute the aggregated information matrix $V_{t-1}$ (updated in Line \ref{line:ndb_update_matrix}) using Neural Tangent Kernel (NTK) features:
\begin{equation}
    V_{t-1} = \frac{\lambda}{\kappa_\mu} \mathbf{I} + \frac{1}{m} \sum_{s=1}^{t-1} g'_s(\theta_0) g'_s(\theta_0)^\top,
    \label{eq:ndb_matrix}
\end{equation}
where $g'_s(\theta_0) \triangleq g(x_{s,1}; \theta_0) - g(x_{s,2}; \theta_0)$ denotes the gradient difference at initialization.

\textbf{Arm Recommendation Strategy (Line \ref{line:ndb_arm1}-Line \ref{line:ndb_arm2}).}
Leveraging the trained network $h(\cdot;\theta_t)$ and the information matrix $V_{t-1}$, NDB-DF recommends two arms as follows:
\begin{itemize}
    \setlength\abovedisplayskip{4pt}
    \setlength\belowdisplayskip{4pt}

    \item \textbf{First Arm Selection (Line \ref{line:ndb_arm1}).} The first arm is selected greedily based on the current neural point estimate to maximize the predicted reward:
    \begin{equation}
        x_{t,1} = \argmax_{x\in\mathcal{X}_t} h(x;\theta_t).
        \label{eq:ndb_arm1}
    \end{equation}

    \item \textbf{Second Arm Selection (Line \ref{line:ndb_arm2}).} The second arm is selected via an optimistic strategy to encourage exploration. We choose the arm that maximizes the Upper Confidence Bound (UCB) of the preference difference relative to $x_{t,1}$:
    \begin{equation}
        x_{t,2} = \argmax_{x\in\mathcal{X}_t}
        \Bigl[h(x;\theta_t) + \nu_T\,\sigma_{t-1}(x,x_{t,1})\Bigr].
        \label{eq:ndb_arm2}
    \end{equation}
    Here, the scaling factor $\nu_T$ is explicitly defined as:
    \begin{equation}
        \nu_T \triangleq \left(\beta_T + B\sqrt{\frac{\lambda}{\kappa_\mu}} + 1\right)\frac{\kappa_\mu}{\lambda},
    \end{equation}
    where $\beta_T$ represents the confidence radius.
    The exploration bonus $\sigma_{t-1}$ quantifies the epistemic uncertainty of the preference difference in the NTK feature space:
    \begin{equation}
        \sigma_{t-1}(x,x') \triangleq
        \sqrt{\frac{\lambda}{\kappa_\mu}}
        \left\|
        \frac{1}{\sqrt{m}}\bigl(g(x;\theta_0)-g(x';\theta_0)\bigr)
        \right\|_{V_{t-1}^{-1}} .
        \label{eq:ndb_bonus}
    \end{equation}
\end{itemize}
\textbf{Updating Interaction History (Line \ref{line:ndb_duel}-Line \ref{line:ndb_update_matrix}).}
Similar to the linear setting, the preference feedback $y_t$ is observed after a stochastic delay $D_t$ (Line \ref{line:ndb_duel}) and is added to the training set $\mathcal{O}_{t'}$ only upon arrival (provided $D_t \le M$). The information matrix $V_t$, however, is updated immediately. Crucially, unlike the linear case which uses raw features, NDB-DF updates $V_t$ using the gradient difference at initialization $g'_{t}(\theta_0)$ (Line \ref{line:ndb_update_matrix}) to capture the geometry of the neural tangent kernel independent of feedback latency.

\begin{algorithm*}[t]
   \caption{Neural Dueling Bandits with Delayed Feedback (NDB-DF)}
   \label{alg:ndb_ucb}
   \begin{algorithmic}[1]
        \STATE \textbf{Input:} Confidence $\delta \in (0,1)$, probability $\rho\in (0,1)$,regularization $\lambda > 0$, width $M > 0$. \label{line:ndb_input}
        \STATE \textbf{Initialize:} $V_0 \triangleq \frac{\lambda}{\kappa_\mu} \mathbf{I}$
        \STATE \textbf{Initialize:} History of arrived feedback $\mathcal{D} \leftarrow \emptyset$. \label{line:ndb_init}
        \FOR{$t = 1, \dots, T$} \label{line:ndb_loop_start}
            \STATE Update observed set: $\mathcal{O}_t = \{(s, x_{s,1}, x_{s,2}, y_s) \mid s < t \text{ and feedback arrives at } t\}$. \label{line:ndb_update_set}
            \STATE Let $\mathcal{D}_t$ be the cumulative set of all observations arrived by round $t$. \label{line:ndb_cumulative}
            \STATE Train NN to find $\theta_t$ by minimizing $\mathcal{L}_t(\theta)$ in Eq. (\ref{eq:ndb_loss}). \label{line:ndb_train}
            \STATE Select first arm: $x_{t,1} = \argmax_{x \in \mathcal{X}_t} h(x; \theta_t)$. \label{line:ndb_arm1}
            \STATE Select second arm: $x_{t,2} = \argmax_{x \in \mathcal{X}_t} [h(x; \theta_t) + \nu_T \sigma_{t-1}(x, x_{t,1})]$. \label{line:ndb_arm2}
            \STATE \textbf{Duel:} Play $(x_{t,1}, x_{t,2})$. Feedback $y_t$ is observed after delay $D_t$. \label{line:ndb_duel}
            \STATE Update Matrix: $V_t \leftarrow V_{t-1}+\frac{1}{m} g'_{t-1}(\theta_0) g'_{t-1}(\theta_0)^\top$. \label{line:ndb_update_matrix}
        \ENDFOR \label{line:ndb_loop_end}
   \end{algorithmic}
\end{algorithm*}

\section{Theoretical Analysis}
\label{sec:analysis}

This section is dedicated to the theoretical analysis of our algorithms, where we establish rigorous regret bounds and discuss their implications.
\subsection{Regret Analysis for Linear Setting (LDB-DF)}

The following theorem establishes an upper bound on the cumulative regret for LDB-DF.

\begin{theorem}
\label{thm:linear_regret}
Suppose Assumption \ref{ass:standard} holds. Let $\lambda > \kappa_\mu L^2$. With probability at least $1-\delta$ for some $\delta\in(0,1)$, the cumulative regret of LDB-DF satisfies:
\begin{align}
    R_T &= \tilde{O}\left( \frac{1}{\rho\kappa_\mu} \left( \sqrt{d} + M \right) d \sqrt{T} \right) \nonumber \\
    &= \tilde{O}\left( \frac{M}{\rho\kappa_\mu} d \sqrt{T} \right).
\end{align}
\end{theorem}
The regret bound in Theorem \ref{thm:linear_regret} explicitly characterizes the impact of feedback latency. The term $\frac{M}{\rho}$ acts as a multiplicative factor, underscoring two key intuitions. A smaller observation probability $\rho$ implies that the effective sample size is reduced. Consequently, the variance of the estimator increases, scaling the regret by $1/\rho$. Increasing the threshold $M$ reduces data observation. A larger $M$ implies tolerating longer feedback delays, which necessitates a wider confidence interval to bound the uncertainty, resulting in a linear penalty on $M$. This exacerbates the systematic bias in the estimator, linearly increasing the regret.
 In the standard setting where feedback is immediate and lossless (i.e., $\rho=1$ and $M=1$), our regret bound simplifies to $\tilde{O}(d\sqrt{T}/\kappa_\mu)$. This fully aligns with the state-of-the-art regret bounds for standard linear dueling bandits established in prior work \citep{Saha2021, Bengs2022, Li2024}, demonstrating that LDB-DF is a generalization of existing methods.

 The proof hinges on bounding the estimation error $\|\theta - \theta_t\|_{V_t}$ under delayed feedback. A primary challenge lies in managing the composite error introduced by the IPW estimator. To address this, we employ a novel ``add-and-subtract" decomposition technique. Specifically, we decompose the gradient error into a martingale noise term and a systematic bias term, then analyze them separately.
\subsection{Regret Analysis for Neural Setting (NDB-DF)}

We now extend our analysis to the neural setting.

\begin{theorem}
\label{thm:neural_regret}
Suppose Assumption \ref{ass:neural_assumptions} holds. Let $m$ be sufficiently large. With probability at least $1-\delta$ for some $\delta\in(0,1)$, the cumulative regret of NDB-DF satisfies:
\begin{align}
    R_T &= \tilde{O}\left( \left( \frac{\sqrt{\tilde{d}}}{\rho\kappa_\mu} + B \sqrt{\frac{\lambda}{\kappa_\mu}} + \frac{M}{\kappa_\mu m \rho} \right) \sqrt{\tilde{d} T} \right) \nonumber \\
    &= \tilde{O}\left( \frac{M}{\rho\kappa_\mu} \tilde{d} \sqrt{T} \right),
\end{align}
where $\tilde{d}$ is the effective dimension of the Neural Tangent Kernel (NTK) matrix.
\end{theorem}

 Mirroring the linear case, the neural regret bound scales linearly with the factor $\frac{M}{\rho}$, capturing the dual sources of difficulty inherent to delayed feedback. Specifically, the inverse dependence on $\rho$ reflects the variance inflation in the IPW estimator caused by a reduced effective sample size. Meanwhile, the linear dependence on $M$ arises because tolerating longer feedback delays necessitates a wider confidence interval to bound the uncertainty, thereby imposing a linear penalty. This confirms that the fundamental hardness of the problem---governed by the trade-off between observation likelihood and delay magnitude---persists regardless of the reward function's complexity. When $\rho=1$ and $M=1$, our result simplifies to $\tilde{O}(\tilde{d}\sqrt{T}/\kappa_\mu)$, aligning with the bounds for Neural Dueling Bandits with immediate feedback derived by \citet{Verma2025NDB}.
 The analysis proceeds in the Neural Tangent Kernel (NTK) regime. We employ a crucial ``add-and-subtract" decomposition strategy on the gradient error. This decomposition allows us to explicitly bound the delay-induced bias by $\frac{M}{\kappa_\mu m \rho}$.

\section{Experiments}
\label{sec:experiments}

In this section, we empirically evaluate the performance of our proposed algorithms, LDB-DF (Linear) and NDB-DF (Neural). We compare them against baselines that do not explicitly handle the bias caused by delays, demonstrating the necessity of our IPW-based delay-handling mechanism.

\subsection{Experimental Setup}

We conduct experiments on both controlled synthetic environments and a large-scale real-world prompt optimization benchmark.

\paragraph{Synthetic Environments.}
To rigorously test the theoretical properties of our algorithms, we generate synthetic data with varying reward structures.

\noindent Linear Setting: We generate a ground-truth preference vector $\theta^ \in \mathbb{R}^d$ uniformly at random on the unit sphere. The reward function is strictly linear: $f(x) = \theta^{\top} x$.

\noindent Non-Linear (Neural) Settings: To evaluate the capacity of NDB-DF to capture complex dependencies, we design two non-linear environments: (1) a Quadratic setting where $f(x) = (\theta^\top x)^2$, and (2) a Cubic setting where $f(x) = (\theta^\top x)^3$.
In all synthetic experiments, we set the arm set size $K=20$ and feature dimension $d=20$. The feedback delay $D_t$ is sampled from a Geometric distribution to simulate stochastic network latency.

\textbf{Real-world Data.}
To demonstrate practical applicability, we apply our contextual dueling bandit framework to the task of Automated Prompt Optimization\citep{fernando2023promptbreeder,lin2023use,chen2024instructzero}. The goal is to identify the optimal prompt that elicits the highest quality response from a Large Language Model (LLM).
We conduct 29 instruction induction tasks from \textbf{InstructZero}\cite{chen2024instructzero}, covering a wide range of settings, including sentiment analysis, reasoning, and translation.
Following the experimental protocol of \cite{lin2024prompt}, we formulate the prompt optimization task as a contextual bandit problem. For each dataset, we generate a candidate pool of $K=500$ prompts. The reward of a selected prompt (arm) is calculated as its accuracy on a held-out validation set of size $N=50$.


\subsection{Baselines}

We compare our proposed methods against two categories of baselines to validate the effectiveness of our bias-correction mechanism.

\noindent LDB-Ignore / NDB-Ignore: These are naive baselines that utilize the same underlying model architectures (Linear or Neural) as our proposed methods but simply disregard delayed feedback. Pending observations are treated as missing and are excluded from the training set until they arrive, without any importance weighting.

\noindent LDB-Heuristic / NDB-Heuristic: These baselines adopt the ``hallucination'' strategy adapted from delayed Bayesian Optimization \citep{Verma2022BO}. Instead of ignoring pending feedback, these methods impute the missing outcome using the current model's predicted preference probability (i.e., the posterior mean estimate). This predicted soft label is then substituted into the loss function to update the estimator, thereby utilizing the pending data albeit with a biased proxy.

Comparing against \textit{Ignore} baselines isolates the benefit of utilizing delayed data, while comparing against \textit{Heuristic} baselines demonstrates the necessity of our unbiased IPW correction versus biased imputation.
\subsection{Results and Analysis}
We report the cumulative regret as the primary metric. For the Prompt Optimization task, we present the average cumulative regret across all 29 Instruction-Induction tasks to ensure the results are statistically robust and not driven by outliers.
\paragraph{Linear Setting.}
Figure \ref{fig:linear_results} presents the performance of LDB-DF compared to the baselines.
In the Synthetic environment (Fig. \ref{fig:linear_results}a), LDB-DF significantly outperforms both LDB-Ignore and LDB-Heuristic. Both baselines suffer from significantly higher regret compared to our method. This demonstrates that neither disregarding delayed feedback (LDB-Ignore) nor naively imputing it with biased predictions (LDB-Heuristic) is sufficient for effective learning. Our IPW-based approach effectively bridges this gap by providing an unbiased correction.
In the Real-world Prompt Optimization task (Fig. \ref{fig:linear_results}b), LDB-DF consistently achieves the lowest cumulative regret averaged across the 29 datasets. The shaded error region indicates that this performance gain is statistically significant. This confirms that a linear approximation with proper unbiased delay handling is far more effective than the suboptimal strategies employed by the baselines.
\begin{figure}[t]
    \centering
    \begin{subfigure}[b]{0.9\linewidth}
        \centering
        \includegraphics[width=\linewidth]{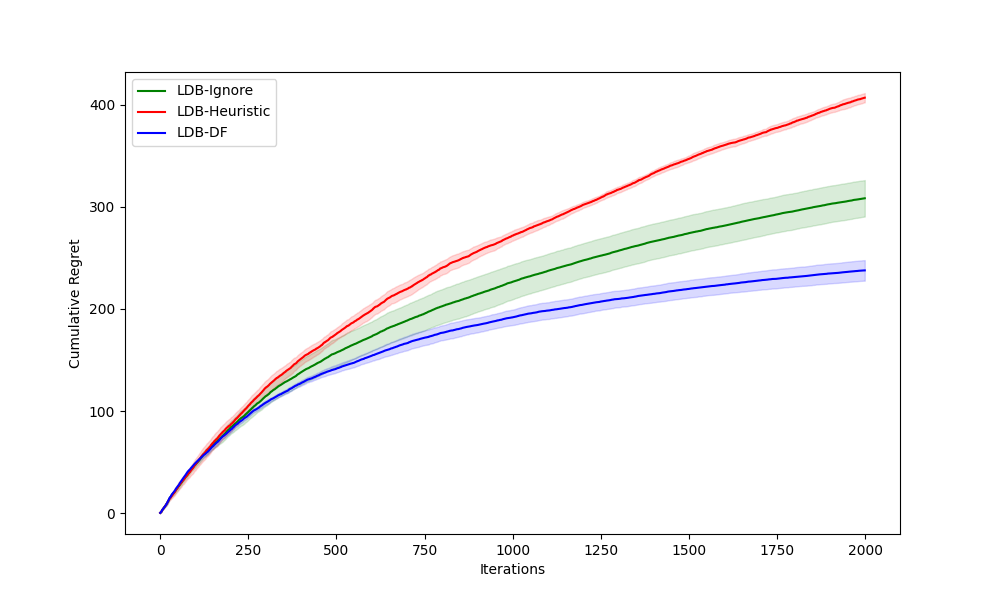}
        \caption{Synthetic ($d=20$)}
        \label{fig:linear_synth}
    \end{subfigure}
    \hfill
    \begin{subfigure}[b]{0.9\linewidth}
        \centering
        \includegraphics[width=\linewidth]{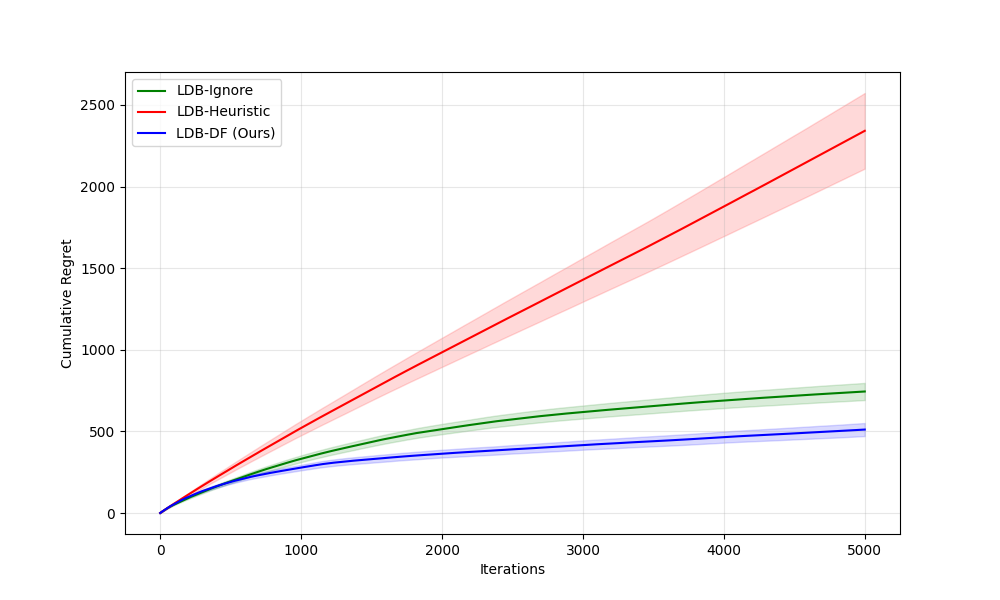}
        \caption{Prompt Opt. (Avg over 29 datasets)}
        \label{fig:linear_real}
    \end{subfigure}
    \caption{\textbf{Linear Results.} Comparison of LDB-DF (Ours) vs. LDB-Ignore. (a) Synthetic data. (b) Aggregated performance on Prompt Optimization.}
    \label{fig:linear_results}
\end{figure}
\paragraph{Neural Setting.}
Figure \ref{fig:neural_results} summarizes the results for NDB-DF.
In the Non-linear Synthetic environments (Fig. \ref{fig:neural_results}a \& \ref{fig:neural_results}b), NDB-DF demonstrates superior convergence speed and lower final regret compared to both NDB-Ignore and NDB-Heuristic in Quadratic and Cubic settings.
Both baselines often converge to suboptimal regions. The naive strategies fail to capture the complex reward landscape: NDB-Ignore suffers from data sparsity, while NDB-Heuristic is misled by self-reinforcing prediction errors (hallucinations). In contrast, NDB-DF remains robust to feedback latency and consistently identifies the optimal arm.
Similarly, in the Prompt Optimization tasks (Fig. \ref{fig:neural_results}c), NDB-DF effectively navigates the semantic search space. By using a neural approximator with our delay-aware objective, it captures non-linear nuances in prompt embeddings and significantly outperforms the biased and naive baselines.

\begin{figure}[t]
    \centering
    
    \begin{subfigure}[b]{0.83\linewidth}
        \centering
        \includegraphics[width=\linewidth]{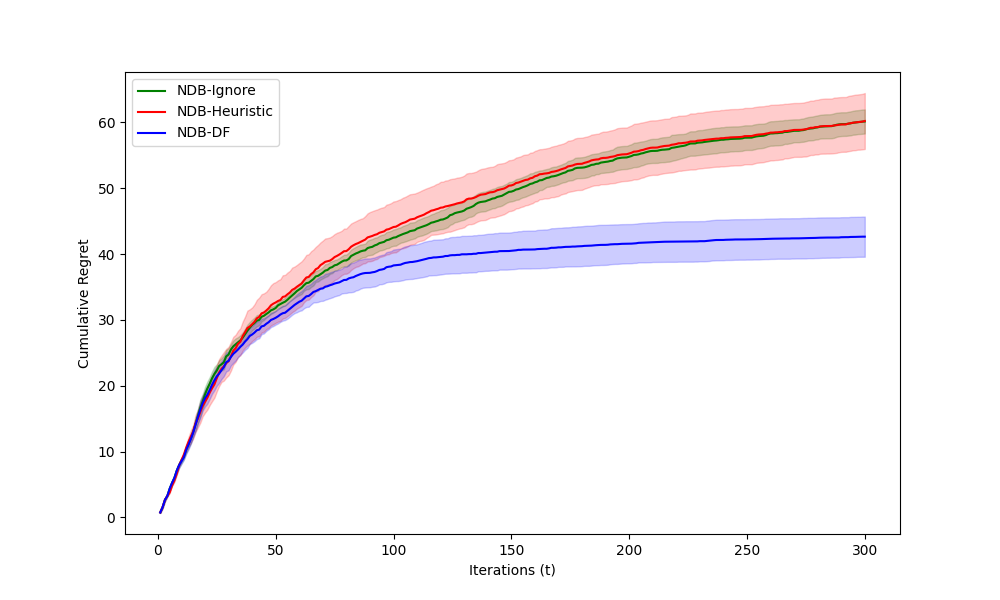}
        \caption{Quadratic}
        \label{fig:neural_quad}
    \end{subfigure}

    \begin{subfigure}[b]{0.83\linewidth}
        \centering
        \includegraphics[width=\linewidth]{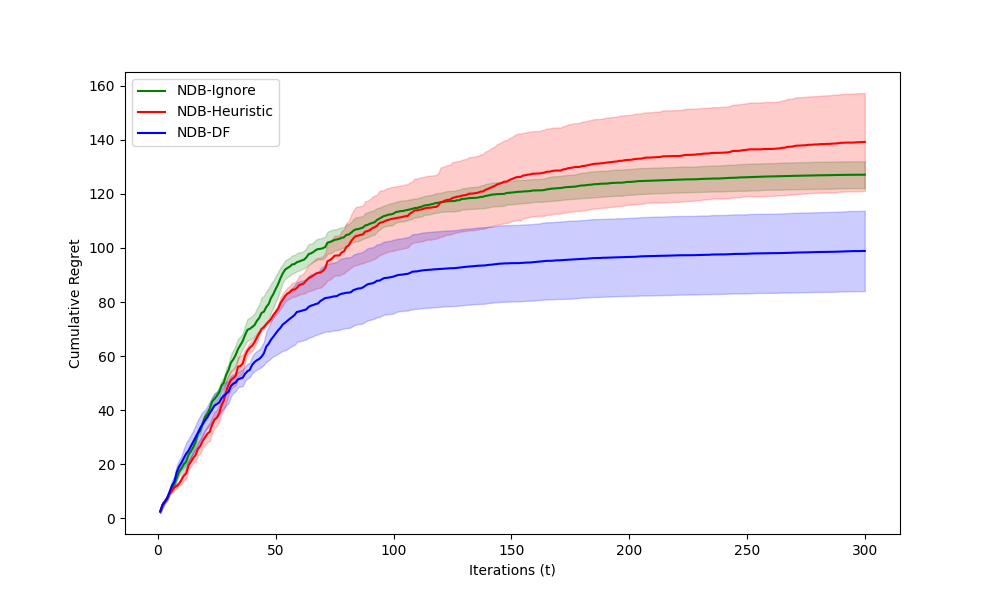}
        \caption{Cubic}
        \label{fig:neural_cubic}
    \end{subfigure}

    \begin{subfigure}[b]{0.83\linewidth}
        \centering
        \includegraphics[width=\linewidth]{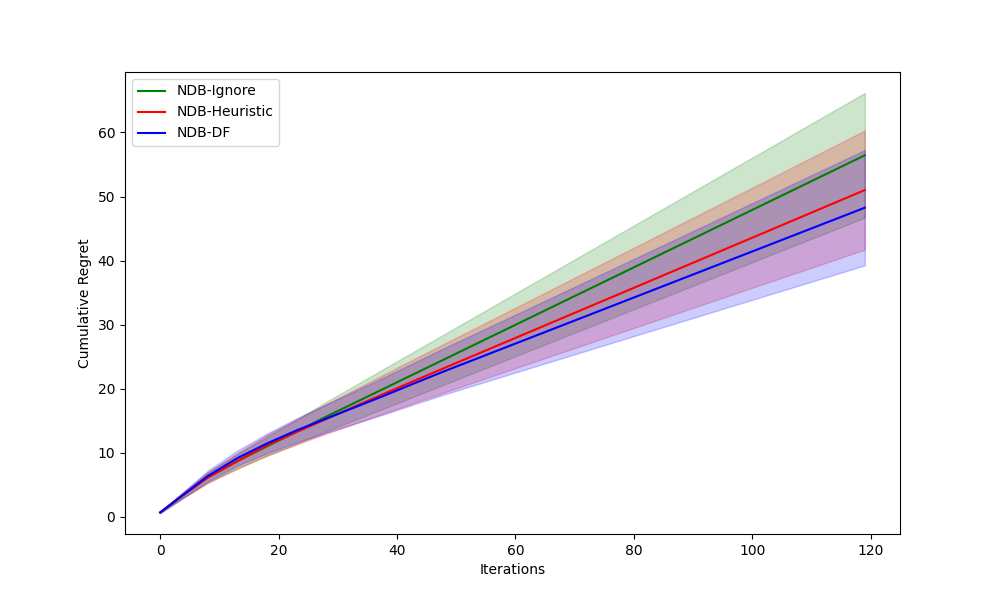}
        \caption{Prompt Opt.}
        \label{fig:neural_real}
    \end{subfigure}

    \caption{\textbf{Neural Results.} Comparison of NDB-DF (Ours) vs. NDB-Ignore. (a) Quadratic synthetic. (b) Cubic synthetic. (c) Aggregated Neural Prompt Optimization.}
    \label{fig:neural_results}
\end{figure}


\section{Related Work}
\label{sec:related_work}

Our work is situated at the intersection of contextual dueling bandits, neural bandits, and online learning with delayed feedback.

\subsection{Contextual Dueling Bandits}
The Dueling Bandit framework, where the learner receives relative preference feedback (e.g., $A \succ B$) rather than absolute rewards, was introduced by \citet{Yue2009, Yue2012} to model scenarios like information retrieval. Early algorithms focused on the $K$-armed setting \citep{Zoghi2014, Komiyama2015}. To handle large or continuous action spaces, the problem was extended to \textit{Contextual} Dueling Bandits, typically assuming a linear relationship between the context vectors and the latent utility function \citep{Saha2021, Bengs2022, Saha2022}. Recently, \citet{Li2024} and \citet{Di2023} have further improved the regret bounds for stochastic contextual dueling bandits, refining the theoretical understanding of preference-based learning. However, all these works operate under the assumption of \textit{immediate} feedback. They fail to account for the stochastic delays inherent in human-in-the-loop systems or complex evaluation pipelines, which can lead to suboptimal decision-making in practice.

\subsection{Neural Bandits}
To address the limitations of linear models in complex real-world applications, Neural Bandits were proposed, utilizing deep neural networks to approximate non-linear reward functions \citep{Zhou2020, Zhang2021}. This direction has seen rapid development, with extensions to various settings such as shallow exploration \citep{Xu2020} and batch learning \citep{Dai2022}. Most relevant to our work is the recent proposal of Neural Dueling Bandits by \citet{Verma2024}, which combines the expressiveness of neural networks with preference-based feedback. While \citet{Verma2024} effectively addresses the non-linearity of reward functions, it does not address the temporal aspect of data collection, specifically the challenges posed by latency. Our NDB-DF algorithm extends this line of research by explicitly modeling and correcting for stochastic delays and missing data within the neural training process, offering a more robust solution for real-world deployment.

\subsection{Learning with Delayed Feedback}
Delayed feedback is a pervasive challenge in online learning, extensively studied in the context of display advertising \citep{Chapelle2014} and conversion optimization \citep{Vernade2017}. \citet{Joulani2013} provided a foundational analysis for general online learning with delays. In the specific context of bandits, \citet{Zhou2019} and \citet{Vernade2017} developed algorithms for Generalized Linear Bandits that handle stochastic delays and missing data (where feedback is lost if it arrives after a threshold). \citet{Garg2020} and \citet{Grover2018} further explored best-arm identification and parallel experimentation under delay.
However, existing delayed bandit algorithms predominantly focus on absolute numerical rewards (standard MAB or Linear Bandits). To the best of our knowledge, no prior work has established a unified framework that simultaneously handles preference feedback (dueling), non-linear function approximation (neural), and stochastic delays/missing. Our LDB-DF and NDB-DF algorithms bridge this gap, providing a comprehensive solution that integrates these three critical components.
\section{Conclusion} \label{sec:conclusion}

In this work, we addressed the practical challenge of Contextual Dueling Bandits under Stochastic Delayed Feedback, a setting pervasive in human-in-the-loop optimization. We formalized the problem and proposed two novel algorithms, LDB-DF and NDB-DF, which leverage Inverse Probability Weighting (IPW) to correct for censorship bias. Our theoretical analysis establishes the first regret bounds for this setting, explicitly characterizing the trade-off between observation latency and sample complexity in both linear and neural regimes.

Empirically, our methods demonstrate significant robustness against stochastic delays, consistently outperforming baselines across synthetic environments and large-scale Prompt Optimization benchmarks. These results validate our framework as a principled solution for asynchronous preference learning. Future directions include extending this approach to general delay distributions and exploring applications to preference-based reinforcement learning with feedback delays.
\section*{Impact Statement}
This paper presents work whose goal is to advance the field of Machine Learning, specifically in sequential decision-making under realistic constraints like latency. Our proposed framework for Prompt Optimization could significantly enhance the efficiency of Large Language Models in various downstream tasks. There are many potential societal consequences of our work, none of which we feel must be specifically highlighted here.







\newpage
\appendix
\onecolumn
\raggedbottom

\section{Overview and Notation}
\label{app:overview}

This appendix provides complete proofs of the theoretical results stated in the main text. We organize the material as follows:
\begin{itemize}
    \item \Cref{app:linear_proofs}: Full proof of the regret bound for the linear setting (Theorem~\ref{thm:linear_regret}).
    \item \Cref{app:neural_proofs}: Full proof of the regret bound for the neural setting (Theorem~\ref{thm:neural_regret}).
\end{itemize}

\paragraph{Notation Recap.}
For the reader's convenience, we summarize the key notation used throughout the proofs:
\begin{center}
\renewcommand{\arraystretch}{1.3}
\begin{tabular}{c l}
\toprule
\textbf{Symbol} & \textbf{Definition} \\
\midrule
$\theta$ & True (unknown) reward parameter \\
$\theta_t$ & Estimated parameter at round $t$ (via IPW-weighted MLE) \\
$\phi(x)$ & Feature map of arm $x$ \\
$\widetilde{\phi}_s$ & Pairwise feature difference: $\phi(x_{s,1}) - \phi(x_{s,2})$ \\
$\mu(\cdot)$ & Logistic link function: $\mu(z) = 1/(1+e^{-z})$ \\
$\dot{\mu}(\cdot)$ & Derivative of the link function \\
$\kappa_\mu$ & Lower bound on $|\dot{\mu}|$ (Assumption~\ref{ass:standard}, A2) \\
$L_\mu$ & Lipschitz constant of $\mu$ (Assumption~\ref{ass:standard}, A1) \\
$L$ & Upper bound on $\|\widetilde{\phi}_s\|_2$ (Assumption~\ref{ass:standard}, A3) \\
$D_s$ & Stochastic delay for round $s$ \\
$M$ & Hard observation threshold \\
$\rho$ & Observation probability: $\mathbb{P}(D_s \le M)$ \\
$\omega_{s,t}$ & IPW weight: $\mathds{1}\{D_s \le \min(M, t-s)\}/\rho$ \\
$V_t$ & Information matrix at round $t$ \\
$\beta_t$ & Confidence radius at round $t$ \\
\bottomrule
\end{tabular}
\end{center}

\paragraph{Key Technical Tool.}
A central ingredient in both the linear and neural proofs is the self-normalized martingale inequality due to \citet{Abbasi2011}. We state it here for reference:
\begin{theorem}[Self-Normalized Bound {\citep[Theorem~1]{Abbasi2011}}]
\label{thm:abbasi_yadkori}
Let $\{\mathcal{F}_t\}_{t=0}^{\infty}$ be a filtration. Let $\{\eta_t\}_{t=1}^{\infty}$ be a real-valued stochastic process adapted to $\{\mathcal{F}_t\}$ such that $\eta_t$ is conditionally $R$-sub-Gaussian given $\mathcal{F}_{t-1}$. Let $\{X_t\}_{t=1}^{\infty}$ be an $\mathbb{R}^d$-valued stochastic process with $X_t$ being $\mathcal{F}_{t-1}$-measurable. Define $\bar{V}_t = V + \sum_{s=1}^t X_s X_s^\top$ for a positive definite matrix $V$. Then, for any $\delta > 0$, with probability at least $1 - \delta$:
\begin{equation}
\label{eq:abbasi_main}
    \left\| \sum_{s=1}^t \eta_s X_s \right\|_{\bar{V}_t^{-1}}^2
    \le 2R^2 \log\!\left( \frac{\det(\bar{V}_t)^{1/2} \det(V)^{-1/2}}{\delta} \right).
\end{equation}
\end{theorem}

We also make use of the following standard determinant-trace inequality \citep[Lemma~10]{Abbasi2011}:
\begin{lemma}[Determinant-Trace Inequality {\citep[Lemma~10]{Abbasi2011}}]
\label{lem:det_trace}
If $\|X_s\|_2 \le L$ for all $s$, then $\det(\bar{V}_t) \le \left(\frac{\mathrm{tr}(V) + tL^2}{d}\right)^d$.
\end{lemma}

Additionally, we use the following information gain bound \citep[Lemma~11]{Abbasi2011}:
\begin{lemma}[Information Gain Bound {\citep[Lemma~11]{Abbasi2011}}]
\label{lem:info_gain_bound}
Suppose $\|X_s\|_2 \le L$ and $\|X_s\|_{\bar{V}_{s-1}^{-1}}^2 \le 1$ for all $s \in [T]$. Then:
\begin{equation}
    \sum_{s=1}^T \|X_s\|_{\bar{V}_{s-1}^{-1}}^2 \le 2\log\frac{\det(\bar{V}_T)}{\det(V)}.
\end{equation}
\end{lemma}

\newpage
\section{Proofs for the Linear Setting (LDB-DF)}
\label{app:linear_proofs}

In this section, we provide the complete proof of Theorem~\ref{thm:linear_regret}. The proof proceeds through three key lemmas: (1) a confidence ellipsoid bound showing that the true parameter $\theta$ lies within an ellipsoid centered at the estimate $\theta_t$ (Lemma~\ref{lem:linear_confidence}); (2) a pointwise estimation error bound (Lemma~\ref{lem:linear_pointwise}); and (3) an information gain bound that controls the sum of exploration bonuses (Lemma~\ref{lem:linear_info_gain}).

\paragraph{Proof Strategy.}
The main challenge in our setting, compared to standard linear dueling bandits, is that the MLE is computed from an \emph{incomplete} dataset: only feedback satisfying $D_s \le \min(M, t-s)$ is available at round $t$. Our approach uses an ``add-and-subtract'' decomposition to isolate: (i) a martingale noise term (bounded via Theorem~\ref{thm:abbasi_yadkori}), (ii) a delay-induced bias term (bounded deterministically by $ML/(\rho\kappa_\mu)$), and (iii) a regularization-induced term (bounded by $\sqrt{\lambda\kappa_\mu}/\kappa_\mu$).

\subsection{Confidence Ellipsoid for the Linear Estimator}

The following lemma establishes that, under delayed feedback, the true parameter $\theta$ remains within a confidence ellipsoid centered at the MLE estimate $\theta_t$, with a radius that explicitly accounts for feedback delays.

\begin{lemma}[Confidence Ellipsoid---Linear Setting]
\label{lem:linear_confidence}
Under Assumption~\ref{ass:standard}, define
\begin{equation}
\label{eq:linear_beta_def}
    \beta_t \triangleq \sqrt{2\log(1/\delta) + d\log\!\bigl(1+tL^2\kappa_\mu / (d\lambda)\bigr)} + ML + \sqrt{\lambda\kappa_\mu}.
\end{equation}
Then, with probability at least $1-\delta$, for all $t = 1, \ldots, T$:
\begin{equation}
\label{eq:linear_confidence_bound}
    \|\theta - \theta_t\|_{V_t} \leq \frac{\beta_t}{\rho\kappa_\mu}.
\end{equation}
\end{lemma}

\begin{proof}

Let $\widetilde{\phi}_s \triangleq \phi(x_{s,1}) - \phi(x_{s,2})$ denote the pairwise feature difference and $\tilde{f}_s \triangleq f(x_{s,1}) - f(x_{s,2}) = \theta^\top \widetilde{\phi}_s$ the true preference difference. For any $\theta' \in \mathbb{R}^d$, define the auxiliary function:
\begin{equation}
\label{eq:G_linear_def}
    G_t(\theta') \triangleq \sum_{s=1}^t \bigl(\mu(\theta'{}^\top \widetilde{\phi}_s) - \mu(\theta^\top \widetilde{\phi}_s)\bigr)\widetilde{\phi}_s + \lambda\theta'.
\end{equation}
We first establish that $G_t$ is strongly monotone. For any $\theta'_1, \theta'_2 \in \mathbb{R}^d$, let $\bar{\theta} = \lambda' \theta'_1 + (1-\lambda')\theta'_2$ for some $\lambda' \in (0,1)$. By the mean value theorem applied componentwise:
\begin{align}
    G_t(\theta'_1) - G_t(\theta'_2)
    &= \left[ \sum_{s=1}^t \dot{\mu}(\bar{\theta}^\top \widetilde{\phi}_s)\,\widetilde{\phi}_s\widetilde{\phi}_s^\top + \lambda \mathbf{I} \right] (\theta'_1 - \theta'_2). \label{eq:linear_mvt}
\end{align}
Since Assumption~\ref{ass:standard} (A2) ensures $\dot{\mu}(\bar{\theta}^\top \widetilde{\phi}_s) \ge \kappa_\mu > 0$ for all $s$, we obtain the lower bound:
\begin{equation}
\label{eq:linear_strong_convexity}
    G_t(\theta'_1) - G_t(\theta'_2)
    \succeq \kappa_\mu \left[ \sum_{s=1}^t \widetilde{\phi}_s \widetilde{\phi}_s^\top + \frac{\lambda}{\kappa_\mu}\mathbf{I} \right] (\theta'_1 - \theta'_2)
    = \kappa_\mu \, V_t \, (\theta'_1 - \theta'_2).
\end{equation}

By construction, $G_t(\theta) = \lambda\theta$ (since the sum telescopes to zero when $\theta' = \theta$). Using the strong monotonicity from \eqref{eq:linear_strong_convexity}:
\begin{align*}
    \|G_t(\theta_t) - \lambda\theta\|_{V_t^{-1}}^2
    &= \|G_t(\theta) - G_t(\theta_t)\|_{V_t^{-1}}^2 \\
    &\ge \bigl(\kappa_\mu V_t(\theta - \theta_t)\bigr)^\top V_t^{-1} \bigl(\kappa_\mu V_t(\theta - \theta_t)\bigr) \\
    &= \kappa_\mu^2 \|\theta - \theta_t\|_{V_t}^2.
\end{align*}
Therefore, by the triangle inequality:
\begin{equation}
\label{eq:linear_error_decomp}
    \|\theta - \theta_t\|_{V_t}
    \le \frac{1}{\kappa_\mu} \|G_t(\theta_t)\|_{V_t^{-1}} + \frac{1}{\kappa_\mu}\|\lambda\theta\|_{V_t^{-1}}.
\end{equation}

Recall that $\theta_t$ is the minimizer of the IPW-regularized loss in \eqref{eq:ldb_loss}. Setting the gradient to zero yields the first-order optimality (MLE) condition:
\begin{equation}
\label{eq:linear_mle_foc}
    \sum_{s=1}^t \left( \mu(\theta_t^\top\widetilde{\phi}_s) - \frac{y_s}{\rho}\mathds{1}\{D_s \le \min(M, t-s)\} \right)\widetilde{\phi}_s + \lambda\theta_t = 0.
\end{equation}
Denoting $f_{t,s} \triangleq \theta_t^\top \widetilde{\phi}_s$, we expand and rearrange:
\begin{align}
    &G_t(\theta_t) \notag\\
    &= \sum_{s=1}^{t} \bigl( \mu(f_{t,s}) - \mu(\tilde{f}_s) \bigr) \widetilde{\phi}_s + \lambda\theta_t \notag\\
    &= \sum_{s=1}^{t} \left( \mu(f_{t,s}) - \frac{1}{\rho} y_s \mathds{1}\{D_s \le M\} + \frac{1}{\rho}\epsilon_s \right) \widetilde{\phi}_s + \lambda\theta_t, \label{eq:linear_Gt_expand}
\end{align}
where we used the identity $\mu(\tilde{f}_s) = \frac{1}{\rho}\bigl(y_s \mathds{1}\{D_s \le M\} - \epsilon_s\bigr)$, with $\epsilon_s \triangleq y_s \mathds{1}\{D_s \le M\} - \mu(\tilde{f}_s)\rho$ being the observation noise.

Adding and subtracting $\frac{1}{\rho} y_s \mathds{1}\{D_s \le \min(M, t-s)\}$ and applying \eqref{eq:linear_mle_foc}:
\begin{align}
    G_t(\theta_t) &= \frac{1}{\rho}\sum_{s=1}^{t} \epsilon_s\,\widetilde{\phi}_s
    + \frac{1}{\rho}\sum_{s=t-M}^{t} y_s \bigl(\mathds{1}\{D_s \le \min(M, t-s)\} - \mathds{1}\{D_s \le M\}\bigr)\widetilde{\phi}_s. \label{eq:linear_three_terms}
\end{align}
where we denote the first summation as Term~(I) (the martingale noise) and the second summation as Term~(II) (the delay-induced bias). The cancellation of the MLE term works as follows: for any $s < t - M$, we have $t - s > M$, so $\min(M, t-s) = M$. Thus, $\mathds{1}\{D_s \le \min(M, t-s)\} = \mathds{1}\{D_s \le M\}$, and the difference in Term~(II) vanishes. Only rounds $s \in [t-M, t]$ contribute to the bias.

For Term~(I),
The noise sequence $\{\epsilon_s\}$ satisfies $\E[\epsilon_s \mid \mathcal{F}_{s-1}] = 0$ and $|\epsilon_s| \le 1$, making it conditionally $1$-sub-Gaussian. Moreover, $\widetilde{\phi}_s$ is $\mathcal{F}_{s-1}$-measurable. Applying Theorem~\ref{thm:abbasi_yadkori} with $R = 1$ and $V = \frac{\lambda}{\kappa_\mu}\mathbf{I}$:
\begin{equation}
\label{eq:linear_noise_bound}
    \left\| \sum_{s=1}^t \epsilon_s \widetilde{\phi}_s \right\|_{V_t^{-1}}^2
    \le 2\log\!\left(\frac{\det(V_t)^{1/2}}{\delta \cdot \det(\frac{\lambda}{\kappa_\mu}\mathbf{I})^{1/2}}\right),
\end{equation}
with probability at least $1-\delta$. By Lemma~\ref{lem:det_trace}, since $\|\widetilde{\phi}_s\|_2 \le L$:
\begin{equation}
\label{eq:linear_det_bound}
    \det(V_t) \le \left(\frac{\lambda/\kappa_\mu + tL^2/d}{1}\right)^d \cdot \left(\frac{1}{\lambda/\kappa_\mu}\right)^{-d},
\end{equation}
which gives:
\begin{equation}
\label{eq:linear_logdet_ratio}
    \log\frac{\det(V_t)^{1/2}}{\det(\frac{\lambda}{\kappa_\mu}\mathbf{I})^{1/2}}
    \le \frac{d}{2}\log\!\left(1 + \frac{tL^2\kappa_\mu}{d\lambda}\right).
\end{equation}
Combining with \eqref{eq:linear_noise_bound}:
\begin{equation}
\label{eq:linear_noise_final}
    \frac{1}{\rho\kappa_\mu}\left\| \sum_{s=1}^t \epsilon_s \widetilde{\phi}_s \right\|_{V_t^{-1}}
    \le \frac{1}{\rho\kappa_\mu}\sqrt{2\log(1/\delta) + d\log\!\bigl(1 + tL^2\kappa_\mu/(d\lambda)\bigr)}.
\end{equation}

For Term~(II),
Since $|y_s| \le 1$ and the difference of indicators satisfies $|\mathds{1}\{D_s \le \min(M,t-s)\} - \mathds{1}\{D_s \le M\}| \le 1$, at most $M$ terms contribute non-trivially. Using $V_t^{-1} \preceq \frac{\kappa_\mu}{\lambda}\mathbf{I}$ and $\|\widetilde{\phi}_s\|_2 \le L$:
\begin{equation}
\label{eq:linear_bias_bound}
    \frac{1}{\rho\kappa_\mu}\left\| \sum_{s=t-M}^{t} \frac{1}{\rho} y_s \bigl(\mathds{1}\{D_s \le \min(M, t-s)\} - \mathds{1}\{D_s \le M\}\bigr)\widetilde{\phi}_s\right\|_{V_t^{-1}}
    \le \frac{ML}{\rho\kappa_\mu}.
\end{equation}

For the regularization term,
Since $V_t \succeq \frac{\lambda}{\kappa_\mu}\mathbf{I}$ implies $V_t^{-1} \preceq \frac{\kappa_\mu}{\lambda}\mathbf{I}$, and using the standard assumption $\|\theta\|_2 \le 1$:
\begin{equation}
\label{eq:linear_reg_bound}
    \frac{1}{\kappa_\mu}\|\lambda\theta\|_{V_t^{-1}}
    = \frac{\lambda}{\kappa_\mu}\sqrt{\theta^\top V_t^{-1} \theta}
    \le \frac{\lambda}{\kappa_\mu}\sqrt{\frac{\kappa_\mu}{\lambda}}\|\theta\|_2
    \le \frac{\sqrt{\lambda\kappa_\mu}}{\kappa_\mu}.
\end{equation}

Combining,
Substituting \eqref{eq:linear_noise_final}, \eqref{eq:linear_bias_bound}, and \eqref{eq:linear_reg_bound} into \eqref{eq:linear_error_decomp}:
\begin{equation}
    \|\theta - \theta_t\|_{V_t}
    \le \frac{1}{\rho\kappa_\mu}\left(\sqrt{2\log(1/\delta) + d\log\!\bigl(1 + tL^2\kappa_\mu/(d\lambda)\bigr)} + ML + \sqrt{\lambda\kappa_\mu}\right)
    = \frac{\beta_t}{\rho\kappa_\mu},
\end{equation}
which completes the proof.
\end{proof}

\subsection{Pointwise Estimation Error}

The confidence ellipsoid immediately implies a pointwise bound on the estimation error for any pair of arms.

\begin{lemma}[Pointwise Estimation Error---Linear Setting]
\label{lem:linear_pointwise}
Under the conditions of Lemma~\ref{lem:linear_confidence}, for any $t = 1, \ldots, T$ and all $x, x' \in \mathcal{X}_t$, with probability at least $1-\delta$:
\begin{equation}
\label{eq:linear_pointwise_bound}
    \bigl| (f(x) - f(x')) - \theta_t^\top (\phi(x) - \phi(x')) \bigr|
    \le \frac{\beta_t}{\rho\kappa_\mu} \left\| \phi(x) - \phi(x') \right\|_{V_{t-1}^{-1}}.
\end{equation}
\end{lemma}

\begin{proof}
Since $f(x) = \theta^\top \phi(x)$ in the linear setting:
\begin{align}
    \bigl| (f(x) - f(x')) - \theta_t^\top (\phi(x) - \phi(x')) \bigr|
    &= \bigl| (\theta - \theta_t)^\top (\phi(x) - \phi(x')) \bigr| \notag\\
    &\le \|\theta - \theta_t\|_{V_{t-1}} \cdot \| \phi(x) - \phi(x') \|_{V_{t-1}^{-1}}, \label{eq:linear_cs_app}
\end{align}
where the inequality follows from Cauchy--Schwarz in the $V_{t-1}$-norm. Applying Lemma~\ref{lem:linear_confidence} yields the desired bound.
\end{proof}

\subsection{Information Gain Bound}

The following lemma bounds the cumulative sum of exploration bonuses, which is essential for converting per-round regret bounds into a cumulative guarantee.

\begin{lemma}[Information Gain---Linear Setting]
\label{lem:linear_info_gain}
Under Assumption~\ref{ass:standard}, choosing $\lambda \ge \kappa_\mu L^2$ ensures:
\begin{equation}
\label{eq:linear_info_gain_bound}
    \sum_{t=1}^T \|\widetilde{\phi}_t\|_{V_{t-1}^{-1}}^2
    \le 2d \log\!\left(1 + \frac{TL^2\kappa_\mu}{d\lambda}\right).
\end{equation}
\end{lemma}

\begin{proof}
Since $V_{t-1} \succeq \frac{\lambda}{\kappa_\mu}\mathbf{I}$, we have $V_{t-1}^{-1} \preceq \frac{\kappa_\mu}{\lambda}\mathbf{I}$. Therefore:
\begin{equation}
    \|\widetilde{\phi}_t\|_{V_{t-1}^{-1}}^2
    \le \frac{\kappa_\mu}{\lambda}\|\widetilde{\phi}_t\|_2^2
    \le \frac{\kappa_\mu L^2}{\lambda}
    \le 1,
\end{equation}
where the last step uses the assumption $\lambda \ge \kappa_\mu L^2$. Since each summand is at most $1$, the inequality $x \le 2\log(1+x)$ is applicable for $x \in [0,1]$. Thus:
\begin{align}
    \sum_{t=1}^T \|\widetilde{\phi}_t\|_{V_{t-1}^{-1}}^2
    &\le \sum_{t=1}^T 2\log\bigl(1 + \|\widetilde{\phi}_t\|_{V_{t-1}^{-1}}^2\bigr) \notag\\
    &= 2\bigl(\log\det V_T - \log\det V_0\bigr) \notag\\
    &= 2\log\frac{\det V_T}{\det V_0} \notag\\
    &\le 2d\log\!\left(1 + \frac{TL^2\kappa_\mu}{d\lambda}\right), \label{eq:linear_sum_bound}
\end{align}
where the second equality uses the matrix determinant lemma (rank-one updates), and the last inequality follows from Lemma~\ref{lem:det_trace} with the bound $\det(V_T) \le (\lambda/\kappa_\mu + TL^2/d)^d$.
\end{proof}

\subsection{Proof of the Main Regret Bound (Theorem~\ref{thm:linear_regret})}
\label{subsec:linear_main_proof}

We now assemble the preceding lemmas to prove the main regret bound.

\begin{theorem}[Restatement of Theorem~\ref{thm:linear_regret}]
\label{thm:linear_regret_app}
Let $\beta_t$ be defined as in \eqref{eq:linear_beta_def} and suppose $\lambda > \kappa_\mu L^2$. With probability at least $1-\delta$:
\begin{equation}
    R_T \le \frac{3}{2} \frac{\beta_T}{\rho\kappa_\mu} \sqrt{2dT\log\!\bigl(1+TL^2\kappa_\mu/(d\lambda)\bigr)}
    = \tilde{O}\!\left(\frac{M}{\rho\kappa_\mu}\,d\sqrt{T}\right).
\end{equation}
\end{theorem}

\begin{proof}

Recall that $2r_t = f(x_t^*) - f(x_{t,1}) + f(x_t^*) - f(x_{t,2})$. We bound each term using Lemma~\ref{lem:linear_pointwise}:
\begin{align}
    2r_t &= \bigl[f(x_t^*) - f(x_{t,1})\bigr] + \bigl[f(x_t^*) - f(x_{t,2})\bigr] \notag\\
    &\le \left[\theta_t^\top(\phi(x_t^*) - \phi(x_{t,1})) + \frac{\beta_t}{\rho\kappa_\mu}\|\phi(x_t^*) - \phi(x_{t,1})\|_{V_{t-1}^{-1}}\right] \notag\\
    &\quad + \left[\theta_t^\top(\phi(x_t^*) - \phi(x_{t,2})) + \frac{\beta_t}{\rho\kappa_\mu}\|\phi(x_t^*) - \phi(x_{t,2})\|_{V_{t-1}^{-1}}\right], \label{eq:linear_rt_step_a}
\end{align}
where we applied Lemma~\ref{lem:linear_pointwise} to both terms (step~(a)).

For the second term, we decompose $\phi(x_t^*) - \phi(x_{t,2}) = (\phi(x_t^*) - \phi(x_{t,1})) + (\phi(x_{t,1}) - \phi(x_{t,2}))$ and apply the triangle inequality to the norm (step~(b)):
\begin{align}
    2r_t &\le 2\theta_t^\top(\phi(x_t^*) - \phi(x_{t,1})) + 2\frac{\beta_t}{\rho\kappa_\mu}\|\phi(x_t^*) - \phi(x_{t,1})\|_{V_{t-1}^{-1}} \notag\\
    &\quad + \theta_t^\top(\phi(x_{t,1}) - \phi(x_{t,2})) + \frac{\beta_t}{\rho\kappa_\mu}\|\phi(x_{t,1}) - \phi(x_{t,2})\|_{V_{t-1}^{-1}}. \label{eq:linear_rt_step_b}
\end{align}

By the selection rule for the \emph{second arm} (Line~\ref{line:ldb_select_arm2} of Algorithm~\ref{alg:ldb_df}), $x_{t,2}$ maximizes the UCB over $\mathcal{X}_t$. In particular, since $x_t^* \in \mathcal{X}_t$:
\begin{equation}
    \theta_t^\top(\phi(x_{t,2}) - \phi(x_{t,1})) + \frac{\beta_t}{\rho\kappa_\mu}\|\phi(x_{t,2}) - \phi(x_{t,1})\|_{V_{t-1}^{-1}}
    \ge \theta_t^\top(\phi(x_t^*) - \phi(x_{t,1})) + \frac{\beta_t}{\rho\kappa_\mu}\|\phi(x_t^*) - \phi(x_{t,1})\|_{V_{t-1}^{-1}}.
\end{equation}
This implies that the first two lines in \eqref{eq:linear_rt_step_b} are bounded by the UCB value of $x_{t,2}$ (step~(c)):
\begin{align}
    2r_t &\le 2\theta_t^\top(\phi(x_{t,2}) - \phi(x_{t,1})) + 2\frac{\beta_t}{\rho\kappa_\mu}\|\phi(x_{t,2}) - \phi(x_{t,1})\|_{V_{t-1}^{-1}} \notag\\
    &\quad + \theta_t^\top(\phi(x_{t,1}) - \phi(x_{t,2})) + \frac{\beta_t}{\rho\kappa_\mu}\|\phi(x_{t,1}) - \phi(x_{t,2})\|_{V_{t-1}^{-1}} \notag\\
    &= \theta_t^\top(\phi(x_{t,2}) - \phi(x_{t,1})) + 3\frac{\beta_t}{\rho\kappa_\mu}\|\phi(x_{t,1}) - \phi(x_{t,2})\|_{V_{t-1}^{-1}}. \label{eq:linear_rt_step_c}
\end{align}

By the selection rule for the \emph{first arm} (Line~\ref{line:ldb_select_arm1}), $x_{t,1} = \argmax_{x} \theta_t^\top \phi(x)$, so $\theta_t^\top \phi(x_{t,1}) \ge \theta_t^\top \phi(x_{t,2})$, which means $\theta_t^\top(\phi(x_{t,2}) - \phi(x_{t,1})) \le 0$ (step~(d)). Therefore:
\begin{equation}
\label{eq:linear_rt_final}
    2r_t \le 3\frac{\beta_t}{\rho\kappa_\mu}\|\phi(x_{t,1}) - \phi(x_{t,2})\|_{V_{t-1}^{-1}}.
\end{equation}

Using $\beta_t \le \beta_T$ (since $\beta_t$ is non-decreasing) and applying the Cauchy--Schwarz inequality:
\begin{align}
    R_T = \sum_{t=1}^T r_t
    &\le \frac{3}{2}\frac{\beta_T}{\rho\kappa_\mu} \sum_{t=1}^T \|\widetilde{\phi}_t\|_{V_{t-1}^{-1}} \notag\\
    &\le \frac{3}{2}\frac{\beta_T}{\rho\kappa_\mu} \sqrt{T \sum_{t=1}^T \|\widetilde{\phi}_t\|_{V_{t-1}^{-1}}^2} \notag\\
    &\le \frac{3}{2}\frac{\beta_T}{\rho\kappa_\mu} \sqrt{2dT\log\!\bigl(1+TL^2\kappa_\mu/(d\lambda)\bigr)}, \label{eq:linear_RT_final}
\end{align}
where the last step applies Lemma~\ref{lem:linear_info_gain}.

Substituting $\beta_T = \sqrt{2\log(1/\delta) + d\log(1 + TL^2\kappa_\mu/(d\lambda))} + ML + \sqrt{\lambda\kappa_\mu}$ and absorbing logarithmic factors into the $\tilde{O}$ notation:
\begin{equation}
    R_T = \tilde{O}\!\left(\frac{1}{\rho\kappa_\mu}(\sqrt{d} + M)\,d\sqrt{T}\right) = \tilde{O}\!\left(\frac{M}{\rho\kappa_\mu}\,d\sqrt{T}\right),
\end{equation}
where the final simplification holds because $M \ge 1$ (a delay threshold of at least one round is necessary). This completes the proof.
\end{proof}

\newpage
\section{Proofs for the Neural Setting (NDB-DF)}
\label{app:neural_proofs}

In this section, we provide the complete proof of Theorem~\ref{thm:neural_regret}. The analysis proceeds within the Neural Tangent Kernel (NTK) regime, where sufficiently wide networks behave approximately as linear functions in the gradient feature space. We leverage this linearization to extend the confidence ellipsoid approach from the linear setting, while carefully controlling the approximation errors that arise from the non-linearity of the neural network.

\subsection{Preliminaries: Neural Tangent Kernel Framework}
\label{subsec:ntk_prelims}

We begin by recalling the key definitions and properties of the NTK framework that underpin our analysis.

\paragraph{NTK Matrix Construction.}
Let $\{x_n\}_{n=1}^{TK}$ denote the set of all context-arm feature vectors encountered across $T$ rounds with $K$ arms per round: $\{x_{t,a}\}_{1 \le t \le T, 1 \le a \le K}$, indexed as $n = K(t-1) + a$. The Neural Tangent Kernel matrix $\mathbf{H} \in \mathbb{R}^{TK \times TK}$ is constructed recursively as follows. Define:
\begin{align}
    \Sigma_{p,q}^{(1)} &= \langle x_p, x_q \rangle, \label{eq:ntk_sigma1}\\
    \mathbf{A}_{p,q}^{(l)} &= \begin{pmatrix} \Sigma_{p,p}^{(l)} & \Sigma_{p,q}^{(l)} \\ \Sigma_{q,p}^{(l)} & \Sigma_{q,q}^{(l)} \end{pmatrix}, \label{eq:ntk_A}\\
    \Sigma_{p,q}^{(l+1)} &= 2\,\E_{(u,v) \sim \mathcal{N}(0, \mathbf{A}_{p,q}^{(l)})} \bigl[\max\{u,0\}\max\{v,0\}\bigr], \label{eq:ntk_sigma_rec}\\
    \tilde{\mathbf{H}}_{p,q}^{(1)} &= \Sigma_{p,q}^{(1)}, \label{eq:ntk_H1}\\
    \tilde{\mathbf{H}}_{p,q}^{(l+1)} &= 2\,\tilde{\mathbf{H}}_{p,q}^{(l)} \cdot \E_{(u,v) \sim \mathcal{N}(0, \mathbf{A}_{p,q}^{(l)})} \bigl[\mathds{1}\{u \ge 0\}\mathds{1}\{v \ge 0\}\bigr] + \Sigma_{p,q}^{(l+1)}. \label{eq:ntk_H_rec}
\end{align}
The NTK matrix is then $\mathbf{H} = \bigl(\tilde{\mathbf{H}}_{p,q}^{(L)}\bigr) + \Sigma^{(L)}$, where $L$ is the network depth. Under Assumption~\ref{ass:neural_assumptions}~(B2), this matrix satisfies $\mathbf{H} \succeq \lambda_0 \mathbf{I}$ for some $\lambda_0 > 0$, ensuring the NTK is non-degenerate.

\paragraph{Width Requirements.}
Our analysis requires the network width $m$ to be sufficiently large, specifically:
\begin{equation}
\label{eq:width_conditions}
    m \ge \mathrm{poly}\bigl(T, L, K, 1/\kappa_\mu, L_\mu, 1/\lambda_0, 1/\lambda, \log(1/\delta)\bigr).
\end{equation}
More precisely, $m$ must satisfy the following conditions for some absolute constant $C > 0$:
\begin{align}
    m &\ge C\, T^4 K^4 L^6 \log(T^2 K^2 L/\delta)\,/\,\lambda_0^4, \label{eq:width_cond1}\\
    m(\log m)^{-3} &\ge C\, \kappa_\mu^{-3} \tau^8 L^{21} \lambda^{-5}, \label{eq:width_cond2}\\
    m(\log m)^{-3} &\ge C\, \kappa_\mu^{-3} \tau^{14} L^{21} \lambda^{-11} L_\mu^6, \label{eq:width_cond3}\\
    m(\log m)^{-3} &\ge C\, T^{14} L^{18} \lambda^{-8}, \label{eq:width_cond4}
\end{align}
where $\tau = 2\sqrt{T/(m\lambda)}$ is an upper bound on the parameter drift (see Lemma~\ref{lem:neural_param_bound} below). These polynomial conditions ensure that: (i) the NTK approximation is valid; (ii) gradient perturbation bounds hold; and (iii) linearization errors are negligible relative to the confidence radius.

\paragraph{Effective Dimension.}
Define the pairwise NTK Gram matrix $\mathbf{H}' \triangleq \sum_{t=1}^T \sum_{(i,j) \in \binom{K}{2}} \frac{1}{m} z_{i,j}^t (z_{i,j}^t)^\top$, where $z_{i,j}^t \triangleq g(x_{t,i}; \theta_0) - g(x_{t,j}; \theta_0)$. The \emph{effective dimension} is:
\begin{equation}
\label{eq:eff_dim_def}
    \tilde{d} \triangleq \log\det\!\left(\frac{\kappa_\mu}{\lambda}\mathbf{H}' + \mathbf{I}\right).
\end{equation}
This quantity plays the role of the ambient dimension $d$ in the linear setting and captures the intrinsic complexity of the problem in the NTK feature space.

\subsection{Key Lemmas for the Neural Setting}

We now state and prove the supporting lemmas needed for the main theorem.

\begin{lemma}[Linear Approximation of $f$ {\citep[Lemma~B.3]{Zhang2021}}]
\label{lem:neural_f_linear}
Suppose the width $m$ satisfies \eqref{eq:width_cond1}. Then, with probability at least $1-\delta$, there exists a parameter $\theta_f \in \mathbb{R}^p$ such that for all $x \in \mathcal{X}_t$, $t \in [T]$:
\begin{equation}
\label{eq:f_linear_approx}
    f(x) = \langle g(x; \theta_0),\, \theta_f - \theta_0 \rangle, \quad \text{with} \quad \sqrt{m}\|\theta_f - \theta_0\|_2 \le \sqrt{2\,\mathbf{h}^\top \mathbf{H}^{-1}\mathbf{h}} \triangleq B,
\end{equation}
where $\mathbf{h} = (f(x_1), \ldots, f(x_{TK}))^\top$ is the vector of true rewards and $p$ denotes the total number of network parameters.
\end{lemma}

This lemma ensures that, in the NTK regime, the true reward function $f$ admits an exact linear representation in the gradient feature space $g(x; \theta_0) = \nabla_\theta h(x; \theta_0)$ at initialization.

\begin{lemma}[Parameter Norm Bound]
\label{lem:neural_param_bound}
For all $t \in [T]$, the estimated parameters satisfy:
\begin{equation}
\label{eq:param_norm_bound}
    \|\theta_t - \theta_0\|_2 \le \sqrt{\frac{2t}{m\lambda}} \triangleq \tau.
\end{equation}
\end{lemma}

\begin{proof}
Since $\theta_t$ minimizes $\mathcal{L}_t(\theta)$ defined in \eqref{eq:ndb_loss}, we have $\mathcal{L}_t(\theta_t) \le \mathcal{L}_t(\theta_0)$. From the regularization term:
\begin{equation}
    \frac{\lambda}{2}\|\theta_t - \theta_0\|_2^2 \le \mathcal{L}_t(\theta_t) \le \mathcal{L}_t(\theta_0).
\end{equation}
By Assumption~\ref{ass:neural_assumptions}~(B3) and the initialization scheme, $h(x; \theta_0) = 0$ for all $x$, so $\mu(h(x_{s,1};\theta_0) - h(x_{s,2};\theta_0)) = \mu(0) = 1/2$. Evaluating the loss at $\theta_0$:
\begin{align}
    \mathcal{L}_t(\theta_0)
    &= -\sum_{s=1}^{t-1}\left[\omega_{s,t} y_s \log\mu(0) + (1-\omega_{s,t} y_s)\log\mu(0)\right] + 0 \notag\\
    &= -\sum_{s=1}^{t-1} \log(1/2)
    = (t-1)\log 2
    \le t\log 2. \label{eq:loss_at_init}
\end{align}
Therefore, $\|\theta_t - \theta_0\|_2^2 \le \frac{2t\log 2}{m\lambda} \le \frac{2t}{m\lambda}$. Taking square roots yields the result.
\end{proof}

\begin{lemma}[Gradient Stability]
\label{lem:gradient_stability}
Under the width conditions \eqref{eq:width_cond2}--\eqref{eq:width_cond4}, for absolute constants $C_1, C_3 > 0$ and with probability at least $1-\delta$, the following hold for all $x \in \mathcal{X}_t$, $t \in [T]$:
\begin{align}
    \|g(x; \theta_t)\|_2 &\le C_3 \sqrt{m}\, L, \label{eq:grad_bound}\\
    \|g(x; \theta_0) - g(x; \theta_t)\|_2 &\le C_1\, m^{1/3} \sqrt{\log m} \left(\frac{t}{\lambda}\right)^{1/3} L^{7/2}. \label{eq:grad_diff_bound}
\end{align}
\end{lemma}

\begin{proof}
Lemma~\ref{lem:neural_param_bound} guarantees $\|\theta_t - \theta_0\|_2 \le \tau = 2\sqrt{t/(m\lambda)}$, which satisfies the perturbation radius requirement of the gradient stability results in \citet[Lemmas~B.5 and B.6]{Zhang2021}. Their results then directly yield \eqref{eq:grad_bound} and \eqref{eq:grad_diff_bound}.
\end{proof}

\begin{lemma}[Linearization Error]
\label{lem:linearization_error}
Define $\varepsilon'_{m,t} \triangleq C_2\, m^{-1/6} \sqrt{\log m}\, L^3 (t/\lambda)^{4/3}$ for an absolute constant $C_2 > 0$. Then, for all $t \in [T]$ and $x, x' \in \mathcal{X}_t$:
\begin{equation}
\label{eq:linearization_error_bound}
    \bigl|\langle g(x;\theta_0) - g(x';\theta_0),\, \theta_t - \theta_0\rangle - (h(x;\theta_t) - h(x';\theta_t))\bigr| \le 2\varepsilon'_{m,t}.
\end{equation}
\end{lemma}

\begin{proof}
By the triangle inequality, the left-hand side decomposes as:
\begin{align}
    &\bigl|\langle g(x;\theta_0) - g(x';\theta_0),\, \theta_t - \theta_0\rangle - (h(x;\theta_t) - h(x';\theta_t))\bigr| \notag\\
    &\le \bigl|\langle g(x;\theta_0), \theta_t - \theta_0\rangle - h(x;\theta_t)\bigr| + \bigl|h(x';\theta_t) - \langle g(x';\theta_0), \theta_t - \theta_0\rangle\bigr|. \label{eq:lin_err_triangle}
\end{align}
Each term is the single-point linearization error, which is bounded by $\varepsilon'_{m,t}$ using the gradient stability bounds from Lemma~\ref{lem:gradient_stability} combined with the Taylor remainder estimate (see \citealt[Lemma~B.4]{Zhang2021}). Summing the two terms yields $2\varepsilon'_{m,t}$.
\end{proof}

\subsection{Confidence Ellipsoid for the Neural Estimator}
\label{subsec:neural_confidence}

The following lemma is the neural counterpart of Lemma~\ref{lem:linear_confidence}. It establishes that the true reward function's ``linear proxy'' $\theta_f$ lies within a confidence ellipsoid around the estimated parameters $\theta_t$.

\begin{lemma}[Confidence Ellipsoid---Neural Setting]
\label{lem:neural_confidence}
Define $\beta_T \triangleq \frac{1}{\kappa_\mu}\sqrt{\tilde{d} + 2\log(1/\delta)}$ and recall the information matrix $V_{t-1} = \frac{\lambda}{\kappa_\mu}\mathbf{I} + \frac{1}{m}\sum_{s=1}^{t-1} \phi'_s (\phi'_s)^\top$, where $\phi'_s \triangleq g(x_{s,1};\theta_0) - g(x_{s,2};\theta_0)$. Under the width conditions \eqref{eq:width_conditions}, with probability at least $1-\delta$:
\begin{equation}
\label{eq:neural_confidence_bound}
    \sqrt{m}\,\|\theta_f - \theta_t\|_{V_{t-1}} \le \frac{\beta_T}{\rho} + B\sqrt{\frac{\lambda}{\kappa_\mu}} + 1 + \frac{M}{\kappa_\mu m \rho}, \quad \forall\, t \in [T].
\end{equation}
\end{lemma}

\begin{proof}

For notational convenience, let $\phi'_s \triangleq g(x_{s,1}; \theta_0) - g(x_{s,2}; \theta_0)$ (gradient at initialization), $\tilde{\phi}'_s \triangleq g(x_{s,1}; \theta_t) - g(x_{s,2}; \theta_t)$ (gradient at current estimate), and $\tilde{h}_{s,t} \triangleq h(x_{s,1}; \theta_t) - h(x_{s,2}; \theta_t)$ (predicted preference difference).

Define the auxiliary function for any $\theta' \in \mathbb{R}^p$:
\begin{equation}
\label{eq:neural_G_def}
    G_t(\theta') \triangleq \frac{1}{m}\sum_{s=1}^{t-1} \bigl[\mu(\langle \theta' - \theta_0, \phi'_s\rangle) - \mu(\langle \theta_f - \theta_0, \phi'_s\rangle)\bigr]\phi'_s + \lambda(\theta' - \theta_0).
\end{equation}
By the same mean value theorem argument as in the linear case (using Assumption~\ref{ass:standard}~(A2)):
\begin{equation}
    G_t(\theta'_1) - G_t(\theta'_2) \succeq \kappa_\mu\, V_{t-1}\, (\theta'_1 - \theta'_2),
\end{equation}
where $V_{t-1} = \frac{1}{m}\sum_{s=1}^{t-1} \phi'_s(\phi'_s)^\top + \frac{\lambda}{\kappa_\mu}\mathbf{I}$.

Since $G_t(\theta_f) = \lambda(\theta_f - \theta_0)$ by construction, the same argument as \eqref{eq:linear_error_decomp} yields:
\begin{equation}
\label{eq:neural_error_decomp}
    \|\theta_f - \theta_t\|_{V_{t-1}}
    \le \frac{1}{\kappa_\mu}\|G_t(\theta_t)\|_{V_{t-1}^{-1}} + \frac{1}{\kappa_\mu}\|\lambda(\theta_f - \theta_0)\|_{V_{t-1}^{-1}}.
\end{equation}

The regularization term is bounded using Lemma~\ref{lem:neural_f_linear}:
\begin{equation}
\label{eq:neural_reg_bound}
    \frac{1}{\kappa_\mu}\|\lambda(\theta_f - \theta_0)\|_{V_{t-1}^{-1}}
    \le \sqrt{\frac{\lambda}{\kappa_\mu}}\,\|\theta_f - \theta_0\|_2
    \le \sqrt{\frac{\lambda}{\kappa_\mu}}\,\frac{B}{\sqrt{m}}.
\end{equation}

Denoting $f_{t,s} \triangleq \langle \theta_t - \theta_0, \phi'_s \rangle$ (linearized prediction) and using the identity $\mu(\langle \theta_f - \theta_0, \phi'_s\rangle) = \mu(f(x_{s,1}) - f(x_{s,2}))$ from Lemma~\ref{lem:neural_f_linear}, we expand $G_t(\theta_t)$:
\begin{align}
    G_t(\theta_t)
    &= \frac{1}{m}\sum_{s=1}^{t-1} \bigl[\mu(f_{t,s}) - \mu(f(x_{s,1}) - f(x_{s,2}))\bigr]\phi'_s + \lambda(\theta_t - \theta_0). \label{eq:neural_Gt_raw}
\end{align}

Using the observation model $\mu(f(x_{s,1}) - f(x_{s,2})) = \frac{1}{\rho}(y_s\mathds{1}\{D_s \le M\} - \epsilon_s)$ and adding/subtracting the IPW-weighted observation $\frac{1}{\rho}y_s\mathds{1}\{D_s \le \min(M, t-s-1)\}$, we obtain:
\begin{align}
    G_t(\theta_t) &= \frac{1}{m\rho}\sum_{s=1}^{t-1} \epsilon_s\, \phi'_s
    + \frac{1}{m\rho}\sum_{s=1}^{t-1} y_s\bigl(\mathds{1}\{D_s \le \min(M,t\!-\!s\!-\!1)\} - \mathds{1}\{D_s \le M\}\bigr)\phi'_s \notag\\
    &\quad + \frac{1}{m}\sum_{s=1}^{t-1}\bigl(\mu(f_{t,s}) - \frac{1}{\rho}y_s\mathds{1}\{D_s \le \min(M,t\!-\!s\!-\!1)\}\bigr)\phi'_s + \lambda(\theta_t - \theta_0). \label{eq:neural_Gt_decomp}
\end{align}
We refer to the three groups above as Term~(I) (martingale noise), Term~(II) (delay bias), and Term~(III) (MLE residual with linearization error), respectively.

The first-order optimality condition for the NDB-DF loss \eqref{eq:ndb_loss} states:
\begin{equation}
\label{eq:neural_mle_foc}
    \frac{1}{m}\sum_{s=1}^{t-1}\bigl(\mu(\tilde{h}_{s,t}) - \frac{1}{\rho}y_s\mathds{1}\{D_s \le \min(M, t-s-1)\}\bigr)\tilde{\phi}'_s + \lambda(\theta_t - \theta_0) = 0.
\end{equation}
Note that this condition involves $\tilde{\phi}'_s$ (gradient at $\theta_t$) and $\tilde{h}_{s,t}$ (network output at $\theta_t$), whereas Term~(III) involves $\phi'_s$ (gradient at $\theta_0$) and $f_{t,s}$ (linearized output). We bridge this gap by introducing two auxiliary error terms.

Add and subtract $\tilde{\phi}'_s$ and $\mu(\tilde{h}_{s,t})$ inside Term~(III):
\begin{align}
    \text{(III)} &= \frac{1}{m}\sum_{s=1}^{t-1}\bigl(\mu(f_{t,s}) - \frac{1}{\rho}y_s\mathds{1}\{D_s \le \min(M, t-s-1)\}\bigr)(\phi'_s - \tilde{\phi}'_s) \notag\\
    &\quad + \frac{1}{m}\sum_{s=1}^{t-1}\bigl(\mu(f_{t,s}) - \mu(\tilde{h}_{s,t})\bigr)\tilde{\phi}'_s \notag\\
    &\quad + \frac{1}{m}\sum_{s=1}^{t-1}\bigl(\mu(\tilde{h}_{s,t}) - \frac{1}{\rho}y_s\mathds{1}\{D_s \le \min(M, t-s-1)\}\bigr)\tilde{\phi}'_s + \lambda(\theta_t - \theta_0) \notag\\
    &\triangleq A_1 + A_2 + 0, \label{eq:neural_A1_A2}
\end{align}
where the last line equals zero by the MLE first-order condition~\eqref{eq:neural_mle_foc}. Therefore, $\text{(III)} = A_1 + A_2$.

We bound each term separately:

For $\|A_1\|_2$,
Using $|\mu(f_{t,s}) - \frac{1}{\rho}y_s\mathds{1}\{D_s \le \min(M, t-s-1)\}| \le 1$ and Lemma~\ref{lem:gradient_stability}:
\begin{align}
    \|A_1\|_2
    &\le \frac{1}{m}\sum_{s=1}^{t-1} \|\phi'_s - \tilde{\phi}'_s\|_2 \notag\\
    &\le \frac{1}{m}\sum_{s=1}^{t-1} 2C_1\, m^{1/3}\sqrt{\log m}\left(\frac{t}{\lambda}\right)^{1/3} L^{7/2} \notag\\
    &= 2C_1\, m^{-2/3}\sqrt{\log m}\, t^{4/3}\, \lambda^{-1/3} L^{7/2}. \label{eq:A1_bound}
\end{align}

For $\|A_2\|_2$,
By the Lipschitz property of $\mu$ (with constant $L_\mu$) and Lemma~\ref{lem:linearization_error}:
\begin{align}
    |\mu(f_{t,s}) - \mu(\tilde{h}_{s,t})|
    &\le L_\mu |f_{t,s} - \tilde{h}_{s,t}|
    \le 2L_\mu\, C_2\, m^{-1/6}\sqrt{\log m}\, L^3\left(\frac{t}{\lambda}\right)^{4/3}. \label{eq:mu_diff_bound}
\end{align}
Combined with $\|\tilde{\phi}'_s\|_2 \le 2C_3\sqrt{m}\, L$ from \eqref{eq:grad_bound}:
\begin{align}
    \|A_2\|_2
    &\le \frac{1}{m}\sum_{s=1}^{t-1} |\mu(f_{t,s}) - \mu(\tilde{h}_{s,t})| \cdot \|\tilde{\phi}'_s\|_2 \notag\\
    &\le 4L_\mu C_2 C_3\, m^{-2/3}\sqrt{\log m}\, t^{7/3}\, L^{7/2}\, \lambda^{-4/3}. \label{eq:A2_bound}
\end{align}

Since $V_{t-1} \succeq \frac{\lambda}{\kappa_\mu}\mathbf{I}$ implies $\|v\|_{V_{t-1}^{-1}} \le \sqrt{\kappa_\mu/\lambda}\,\|v\|_2$, we have:
\begin{align}
    \frac{1}{\kappa_\mu}\|\text{(III)}\|_{V_{t-1}^{-1}}
    &\le \frac{1}{\kappa_\mu}\sqrt{\frac{\kappa_\mu}{\lambda}}\bigl(\|A_1\|_2 + \|A_2\|_2\bigr). \label{eq:term_III_bound}
\end{align}
Under the width conditions \eqref{eq:width_cond2}--\eqref{eq:width_cond4}, both $\|A_1\|_2$ and $\|A_2\|_2$ are $o(1/\sqrt{m})$, so their combined contribution satisfies:
\begin{equation}
\label{eq:term_III_final}
    \frac{\sqrt{m}}{\kappa_\mu}\|\text{(III)}\|_{V_{t-1}^{-1}} \le 1.
\end{equation}

The noise $\epsilon_s = y_s\mathds{1}\{D_s \le M\} - \mu(f(x_{s,1}) - f(x_{s,2}))\rho$ is conditionally $1$-sub-Gaussian (since $|\epsilon_s| \le 1$ and $\E[\epsilon_s \mid \mathcal{F}_{s-1}] = 0$). The vectors $\frac{1}{\sqrt{m}}\phi'_s$ are $\mathcal{F}_{s-1}$-measurable. Applying Theorem~\ref{thm:abbasi_yadkori} to the sequence $\{\epsilon_s, \frac{1}{\sqrt{m}}\phi'_s\}$:
\begin{equation}
\label{eq:neural_noise_bound}
    \left\| \sum_{s=1}^{t-1} \epsilon_s \frac{\phi'_s}{\sqrt{m}} \right\|_{V_{t-1}^{-1}}^2
    \le 2\log(1/\delta) + \log\frac{\det(V_{t-1})}{\det(V_0)}.
\end{equation}
Using the effective dimension bound: $\log\frac{\det(V_{t-1})}{\det(V_0)} \le \log\det\bigl(\frac{\kappa_\mu}{\lambda}\mathbf{H}' + \mathbf{I}\bigr) = \tilde{d}$, we obtain:
\begin{equation}
\label{eq:neural_noise_final}
    \frac{1}{\rho\kappa_\mu}\left\| \sum_{s=1}^{t-1} \epsilon_s \frac{\phi'_s}{\sqrt{m}} \right\|_{V_{t-1}^{-1}}
    \le \frac{1}{\rho\kappa_\mu}\sqrt{\tilde{d} + 2\log(1/\delta)}
    = \frac{\beta_T}{\rho}.
\end{equation}

By the same argument as in the linear setting, for $s < t - 1 - M$ the difference of indicators vanishes. At most $M$ terms contribute, each bounded by $\|\phi'_s/\sqrt{m}\|_{V_{t-1}^{-1}} \le \sqrt{\kappa_\mu/\lambda}$. Therefore:
\begin{equation}
\label{eq:neural_bias_bound}
    \frac{\sqrt{m}}{\kappa_\mu m\rho}\left\| \sum_{s} y_s\bigl(\mathds{1}\{D_s \le \min(M,t-s-1)\} - \mathds{1}\{D_s \le M\}\bigr)\phi'_s \right\|_{V_{t-1}^{-1}} \le \frac{M}{\kappa_\mu m\rho}.
\end{equation}

Multiplying \eqref{eq:neural_error_decomp} by $\sqrt{m}$ and combining the above bounds:
\begin{align}
    \sqrt{m}\,\|\theta_f - \theta_t\|_{V_{t-1}}
    &\le \frac{\beta_T}{\rho} + \frac{M}{\kappa_\mu m\rho} + 1 + B\sqrt{\frac{\lambda}{\kappa_\mu}},
\end{align}
which completes the proof.
\end{proof}

\subsection{Pointwise Estimation Error and Regret Bound}

We can now prove the main theorem for the neural setting.

\begin{theorem}[Neural Pointwise Error and Regret---Restatement of Theorem~\ref{thm:neural_regret}]
\label{thm:neural_regret_app}
Under the width conditions \eqref{eq:width_conditions} and Assumptions~\ref{ass:standard}--\ref{ass:neural_assumptions}, with probability at least $1-\delta$, for all $x, x' \in \mathcal{X}_t$ and $t \in [T]$:
\begin{equation}
\label{eq:neural_pointwise}
    |[f(x) - f(x')] - [h(x;\theta_t) - h(x';\theta_t)]|
    \le \nu_T\, \sigma_{t-1}(x, x') + 2\varepsilon'_{m,t},
\end{equation}
where $\nu_T = (\beta_T/\rho + B\sqrt{\lambda/\kappa_\mu} + 1 + M/(\kappa_\mu m\rho)) \cdot \kappa_\mu/\lambda$ and $\sigma_{t-1}$ is defined in \eqref{eq:ndb_bonus}. Furthermore, the cumulative regret satisfies:
\begin{equation}
    R_T = \tilde{O}\!\left(\left(\frac{\sqrt{\tilde{d}}}{\rho\kappa_\mu} + B\sqrt{\frac{\lambda}{\kappa_\mu}} + \frac{M}{\kappa_\mu m\rho}\right)\sqrt{\tilde{d}\, T}\right)
    = \tilde{O}\!\left(\frac{M}{\rho\kappa_\mu}\,\tilde{d}\sqrt{T}\right).
\end{equation}
\end{theorem}

\begin{proof}

Using the linearization from Lemma~\ref{lem:neural_f_linear}, $f(x) - f(x') = \langle \phi'_{x,x'},\, \theta_f - \theta_0\rangle$ where $\phi'_{x,x'} \triangleq g(x;\theta_0) - g(x';\theta_0)$. By the triangle inequality and Lemma~\ref{lem:linearization_error}:
\begin{align}
    &|f(x) - f(x') - (h(x;\theta_t) - h(x';\theta_t))| \notag\\
    &\le |\langle \phi'_{x,x'},\, \theta_f - \theta_t\rangle| + |\langle \phi'_{x,x'},\, \theta_t - \theta_0\rangle - (h(x;\theta_t) - h(x';\theta_t))|. \label{eq:neural_pw_decomp}
\end{align}
The first term is controlled by the confidence ellipsoid (Lemma~\ref{lem:neural_confidence}), and the second term is bounded by the linearization error $2\varepsilon'_{m,t}$ (Lemma~\ref{lem:linearization_error}). For the first term, applying Cauchy--Schwarz and Lemma~\ref{lem:neural_confidence}:
\begin{align}
    |\langle \phi'_{x,x'},\, \theta_f - \theta_t\rangle|
    &= \left|\left\langle \frac{\phi'_{x,x'}}{\sqrt{m}},\, \sqrt{m}(\theta_f - \theta_t)\right\rangle\right| \notag\\
    &\le \left\|\frac{\phi'_{x,x'}}{\sqrt{m}}\right\|_{V_{t-1}^{-1}} \cdot \sqrt{m}\,\|\theta_f - \theta_t\|_{V_{t-1}} \notag\\
    &\le \sigma_{t-1}(x,x') \cdot \frac{\lambda/\kappa_\mu}{\sqrt{\lambda/\kappa_\mu}} \cdot \left(\frac{\beta_T}{\rho} + B\sqrt{\frac{\lambda}{\kappa_\mu}} + 1 + \frac{M}{\kappa_\mu m\rho}\right) \notag\\
    &= \nu_T\, \sigma_{t-1}(x, x'). \label{eq:neural_cs_bound}
\end{align}
Combining \eqref{eq:neural_pw_decomp} and \eqref{eq:neural_cs_bound} yields \eqref{eq:neural_pointwise}.

The analysis mirrors the linear case. For the instantaneous regret $2r_t = f(x_t^*) - f(x_{t,1}) + f(x_t^*) - f(x_{t,2})$, we use the pointwise bound \eqref{eq:neural_pointwise} (step~(a)), the triangle inequality (step~(b)), the Lemma~\ref{lem:linearization_error} to relate the linearized inner product to the network output (step~(c)), and the arm selection rules (steps~(d) and~(e)):
\begin{align}
    2r_t &\le (h(x_{t,2};\theta_t) - h(x_{t,1};\theta_t)) + 3\nu_T\left\|\frac{\phi'_{t,1:2}}{\sqrt{m}}\right\|_{V_{t-1}^{-1}} + 6\varepsilon'_{m,t} \notag\\
    &\le 3\nu_T \left\|\frac{\phi'_{t,1:2}}{\sqrt{m}}\right\|_{V_{t-1}^{-1}} + 6\varepsilon'_{m,t}, \label{eq:neural_rt_final}
\end{align}
where $\phi'_{t,1:2} \triangleq g(x_{t,1};\theta_0) - g(x_{t,2};\theta_0)$ and the last step uses $h(x_{t,1};\theta_t) \ge h(x_{t,2};\theta_t)$ from the first-arm selection rule. The detailed steps (a)--(e) follow the same logical chain as in the linear proof (\Cref{subsec:linear_main_proof}), with the addition of Lemma~\ref{lem:linearization_error} to handle the gap between $f_{t,s}$ and $h(\cdot;\theta_t)$.

Under the width conditions, $6T\varepsilon'_{m,T} \le 1$ (the linearization errors are negligible when the network is sufficiently wide). Therefore:
\begin{align}
    R_T
    &\le \frac{3}{2}\nu_T \sum_{t=1}^T \left\|\frac{\phi'_{t,1:2}}{\sqrt{m}}\right\|_{V_{t-1}^{-1}} + 3T\varepsilon'_{m,T} \notag\\
    &\le \frac{3}{2}\nu_T \sqrt{T \sum_{t=1}^T \left\|\frac{\phi'_{t,1:2}}{\sqrt{m}}\right\|_{V_{t-1}^{-1}}^2} + 1. \label{eq:neural_RT_cs}
\end{align}
By Lemma~\ref{lem:info_gain_bound} (applied with the NTK features $\frac{1}{\sqrt{m}}\phi'_{t,1:2}$):
\begin{equation}
    \sum_{t=1}^T \left\|\frac{\phi'_{t,1:2}}{\sqrt{m}}\right\|_{V_{t-1}^{-1}}^2 \le 2\tilde{d},
\end{equation}
where we used $\log\frac{\det(V_T)}{\det(V_0)} \le \tilde{d}$ and verified the unit bound condition $\|\frac{1}{\sqrt{m}}\phi'_{t,1:2}\|_{V_{t-1}^{-1}}^2 \le 1$ (which follows from $\lambda/\kappa_\mu > 1$ and $\frac{1}{m}\|\phi'_{t,1:2}\|_2^2 \le c_0$ for a constant $c_0$).

Substituting and using $\beta_T = \frac{1}{\kappa_\mu}\sqrt{\tilde{d} + 2\log(1/\delta)}$:
\begin{align}
    R_T
    &\le \frac{3}{2}\left(\frac{\beta_T}{\rho} + B\sqrt{\frac{\lambda}{\kappa_\mu}} + 1 + \frac{M}{\kappa_\mu m\rho}\right)\frac{\kappa_\mu}{\lambda}\sqrt{2\tilde{d}\, T} + 1 \notag\\
    &= \tilde{O}\!\left(\left(\frac{\sqrt{\tilde{d}}}{\rho\kappa_\mu} + B\sqrt{\frac{\lambda}{\kappa_\mu}} + \frac{M}{\kappa_\mu m\rho}\right)\sqrt{\tilde{d}\, T}\right). \label{eq:neural_RT_final}
\end{align}
When the width $m$ is sufficiently large (absorbing the $1/m$ factor), the dominant term is $\frac{\sqrt{\tilde{d}}}{\rho\kappa_\mu}\sqrt{\tilde{d}\, T} = \frac{\tilde{d}\sqrt{T}}{\rho\kappa_\mu}$, multiplied by the delay factor. This yields the simplified bound $R_T = \tilde{O}(\frac{M}{\rho\kappa_\mu}\,\tilde{d}\sqrt{T})$, completing the proof.
\end{proof}

\end{document}